\RequirePackage{fix-cm}
\documentclass[twocolumn]{svjour3}          
\smartqed  

\newcommand{\eg}{{\em e.g.}}
\newcommand{\ie}{{\em i.e.}}

\newcommand{\smoemdash}{{\,\textemdash\,}}

\usepackage{graphicx}
\usepackage{rotating}
\usepackage{multirow}
\usepackage{multicol}
\usepackage{siunitx}
\usepackage{natbib}
\usepackage{epsfig} 
\usepackage{amsfonts}
\usepackage{amssymb,amsmath}
\usepackage{bm}
\usepackage[lofdepth,lotdepth]{subfig}
\usepackage{float}
\usepackage{multirow}
\usepackage{makecell}
\usepackage[pdftex, colorlinks=true,bookmarks=false,citecolor=black,urlcolor=blue,filecolor=black,linkcolor=black]{hyperref}
\usepackage{algorithm,algpseudocode}
\algnewcommand{\Inputs}[1]{%
  \State \textbf{Inputs:}
  \Statex \hspace*{\algorithmicindent}\parbox[t]{.8\linewidth}{\raggedright #1}
}
\algnewcommand{\Initialize}[1]{%
  \State \textbf{Initialize:}
  \Statex \hspace*{\algorithmicindent}\parbox[t]{.8\linewidth}{\raggedright #1}
}
\usepackage{nomencl}
\makenomenclature

\begin{document}
\sloppy

\title{Active Expansion Sampling for Learning Feasible Domains in an Unbounded Input Space\thanks{
This work was funded through a University of Maryland Minta Martin Grant.
}
}


\author{Wei Chen         \and
        Mark Fuge 
}


\institute{W. Chen \at
              Department of Mechanical Engineering \\
              University of Maryland \\
              College Park, MD 20742 \\
              \email{wchen459@umd.edu}           
           \and
           M. Fuge \at
              Department of Mechanical Engineering \\
              University of Maryland \\
              College Park, MD 20742
}

\date{Received: date / Accepted: date}

\maketitle

\begin{abstract}
Many engineering problems require identifying feasible domains under implicit constraints. One example is finding acceptable car body styling designs based on constraints like aesthetics and functionality. Current active-learning based methods learn feasible domains for bounded input spaces. However, we usually lack prior knowledge about how to set those input variable bounds. Bounds that are too small will fail to cover all feasible domains; while bounds that are too large will waste query budget. To avoid this problem, we introduce Active Expansion Sampling (AES), a method that identifies (possibly disconnected) feasible domains over an unbounded input space. AES progressively expands our knowledge of the input space, and uses successive exploitation and exploration stages to switch between learning the decision boundary and searching for new feasible domains. We show that AES has a misclassification loss guarantee within the explored region, independent of the number of iterations or labeled samples. Thus it can be used for real-time prediction of samples' feasibility within the explored region. We evaluate AES on three test examples and compare AES with two adaptive sampling methods\smoemdash the \textit{Neighborhood-Voronoi} algorithm and the \textit{straddle} heuristic\smoemdash that operate over fixed input variable bounds.
\end{abstract}

\mbox{}

\nomenclature{$l$}{Length scale of the Gaussian kernel}
\nomenclature{$\bm{x}$}{Input sample}
\nomenclature{$\bm{x}^*$}{Optimal query}
\nomenclature{$y$}{Label of an input sample (feasible/infeasible)}
\nomenclature{$\hat{y}$}{Label estimated by Gaussian Process}
\nomenclature{$X_L$}{Set of labeled samples}
\nomenclature{$X_U$}{Set of unlabeled samples (candidates)}
\nomenclature{$d$}{Input space dimension}
\nomenclature{$\bar{f}$}{Predictive mean of Gaussian Process}
\nomenclature{$V$}{Predictive variance of Gaussian Process}
\nomenclature{$h$}{Actual evaluation function}
\nomenclature{$f$}{Estimated feasibility function}
\nomenclature{$p_{\epsilon}$}{$\epsilon$-margin probability}
\nomenclature{$\Phi$}{Cumulative distribution function of standard Gaussian distribution}
\nomenclature{$L$}{Misclassification loss}
 
\printnomenclature

\section{Introduction}

In applications like design space exploration~\citep[e.g.][]{yannou2005faster,devanathan2010creating,larson2012design} and reliability analysis~\citep[e.g.][]{lee2008sampling,zhuang2012sequential}, people need to find feasible domains within which solutions are valid. Sometimes the constraints that define those feasible domains are implicit, \ie, they cannot be represented analytically. Examples of these constraints are aesthetics, functionality, or performance requirements, which are usually evaluated by human assessment, experiments, or time-consuming computer simulations. Thus usually it is expensive to detect the feasibility of a given input. In such cases, one would like to use as few samples as possible while still approximating the feasible domain well.

To solve such problems, researchers have used \textit{active learning} (or \textit{adaptive sampling})\footnote{Note that in this paper the terms ``active learning" and ``adaptive sampling" are interchangeable.} to sequentially select the most informative instances and query their feasibility, so that the number of queries can be minimized~\citep{larson2012design,lee2008sampling,zhuang2012sequential,huang2010modified,ren2011design}.
These methods require fixed bounds over the input space, and only pick queries inside those bounds. But what if we do not know how wide to set those bounds? If we set the bounds too large, an active learner will require an excessively large budget to explore the input space; whereas if we set the bounds too small, we cannot guarantee that an algorithm will recover all the feasible domains~\citep{chen2017beyond}. In this case, we need an active learning method that can gradually expand our knowledge about the input space until we have either discovered all feasible domains or used up our remaining query budget.

This paper proposes a method\smoemdash which we call \textit{Active Expansion Sampling} (AES)\smoemdash to solve that problem by casting the detection of feasible domains as an \textit{unbounded domain estimation problem}.
In an unbounded domain estimation problem, given an expensive function $h: \mathcal{X}\in\mathbb{R}^d\rightarrow\{-1,1\}$ that evaluates any point $\bm{x}$ in an \textit{unbounded} input data space $\mathcal{X}$, we want to find (possibly disconnected) feasible domains in which $h(\bm{x})=1$. Specifically, $h$ could be costly computation, time-consuming experiments, or human evaluation, so that the problem cannot be solved analytically.
By \textit{unbounded}, we mean that we don't manually bound the input space. Thus the input space can be considered as infinite, and theoretically if the query budget allows, our method can keep expanding the explored area of the input space. To use as few function evaluations as necessary to identify feasible domains, AES first fully exploits (up to an accuracy threshold) any feasible domains it knows about and then, budget permitting, searches outward to discover other feasible domains.

The main contributions of this paper are:
\begin{enumerate}
\item We introduce the AES method for identifying (possibly disconnected) feasible domains over an unbounded input space.
\item We provide a framework that transfers bounded active learning methods into methods that can operate over unbounded input space.
\item We introduce a dynamic local pool method that efficiently finds near optimal solutions to the global optimization problem (Eq.~\ref{eq:strategy}) for selecting queries.
\item We prove a constant theoretical bound for AES's misclassification error at any iteration inside the explored region.
\end{enumerate}

\section{Background and Related Work}

Essentially, the unbounded domain estimation problem breaks down into two tasks explored by past researchers: 1)~the active learning task, where we efficiently query the feasibility of inputs; and 2)~the classification task, where we estimate decision boundaries (\ie, boundaries of feasible domains) that separates the feasible class and the infeasible class (\ie, feasible regions and infeasible regions). For the first task, we will review relevant past work on active learning. For the second task, we use the Gaussian Process as the classifier in this paper and will introduce basic concepts of Gaussian Processes.

\subsection{Feasible Domain Identification}

Past work in design and optimization has proposed ways to identify feasible domains or decision boundaries of expensive functions.
Generally those methods were proposed to reduce the number of simulation runs and improve the accuracy of surrogate models in simulation-based design and reliability assessment~\citep{lee2008sampling,basudhar2010improved}. Also, the problem of feasible domain identification is also equivalent to estimating the level set or the threshold boundaries of a function, where the feasible/infeasible region becomes superlevel/sublevel set~\citep{bryan2006active,gotovos2013active}. Such methods select samples that are expected to best improve the surrogate model's accuracy. A common rule is to sample on the estimated decision boundary, but not close to existing sample points. Existing methods achieve this by (1)~explicitly optimizing or constraining the decision function or the distance between the new sample and the existing samples~\citep{basudhar2008adaptive,basudhar2010improved,singh2017sequential}, or (2)~selecting points based on the estimated function values and their confidence at candidate points~\citep{lee2008sampling,bryan2006active,gotovos2013active,chen2014local,chen2015important,yang2015active}.

\subsection{Active Learning}
\label{sec:active_learning}

Methods for feasible domain identification usually require strategies that sequentially sample points in an input space, such that the sample size is minimized. These strategies fall under the larger category of active learning.

There are three main scenarios of active learning problems: (1)~membership query synthesis, (2)~stream-based selective sampling, and
(3)~pool-based sampling~\citep{settles2010active}. In the \textit{membership query model}, the learner generates samples de novo for labeling. For classification tasks, researchers have typically applied membership query models to learning finite concept classes~\citep{jackson1997efficient,angluin2004queries,king2004functional,pmlr-v30-Awasthi13} and halfspaces~\citep{alabdulmohsin2015efficient,chen2016near}.
In the \textit{stream-based selective sampling model}, an algorithm draws each unlabeled sample from an incoming data distribution, and then decides whether or not to query that label. This decision can be based on some informativeness measure of the drawn sample~\citep{dagan1995committee,freund1997selective,schohn2000less,cavallanti2009linear,cesa2009robust,orabona2011better,dekel2012selective,agarwal2013selective}, or whether the drawn sample is inside a \textit{region of uncertainty}~\citep{cohn1994improving,dasgupta2009analysis}.
In the \textit{pool-based sampling model}, there is a small pool of labeled samples and a large (but finite) pool of unlabeled samples, where the learner selects new queries from the unlabeled pool. 

The unbounded domain estimation problem assumes that synthesizing an unlabeled sample from the input space is not expensive (as in the membership query scenario), since otherwise we have to use existing samples and the input space will be bounded. An example that satisfies this assumption is experimental design, where we can form an experiment by selecting a set of parameters. With this assumption, our proposed method approximates the pool-based sampling setting by synthesizing a pool of unlabeled samples in each iteration.

A pool-based sampling method first trains a classifier using the labeled samples. Then it ranks the unlabeled samples based on their \textit{informativeness} indicated by an \textit{acquisition function}. A query is then selected from the pool of unlabeled samples according to their rankings. After that, we add the selected query into the set of labeled data and repeat the previous process until our query budget is reached. Many of these methods use the informativeness criteria that select queries with the maximum label ambiguity~\citep{lewis1994sequential,settles2008analysis,huang2010active}, contributing the highest estimated expected classification error~\citep{campbell2000query,zhu2003combining,nguyen2004active,krempl2015optimised}, best reducing the \textit{version space}~\citep{tong2001support}, or where different classifiers disagree the most~\citep{mccallum1998employing,argamon1999committee}. Such methods are usually good at \textit{exploitation}, since they keep querying points close to the decision boundary, refining our estimate of it. 

However, when the input space may have multiple regions of interest (\ie, feasible regions), these methods may not work well if the active learner is not aware of all the regions of interest initially. Note that while some of the methods mentioned above also consider representativeness~\citep{mccallum1998employing,zhu2003combining,nguyen2004active,settles2008analysis,huang2010active}, or the diversity of queries~\citep{hoi2009semisupervised,yang2015multi}, they don't explicitly explore unknown regions and discover other regions of interests. To address this issue, an active learner also has to allow for \textit{exploration} (\ie, to query in unexplored regions where no labeled sample has been seen yet). A learner must trade-off exploitation and exploration. 

To query in an unexplored region, there are methods that (1)~take into account the predictive variance at unlabeled samples when selecting new queries~\citep{bryan2006active,kapoor2010gaussian,gotovos2013active}, (2)~naturally balance exploitation/exploration by looking a the expected error~\citep{mac2014hierarchical}, or (3)~make exploitative and exploratory queries separately using different strategies
~\citep{baram2004online,osugi2005balancing,krause2007nonmyopic,hoang14,bouneffouf2016exponentiated,hsu2015active}. In previous methods, the exploitation-exploration trade-off was performed in a bounded input space or a fixed sampling pool. However, in the unbounded domain estimation problem, there is no fixed sampling pool and we are usually uncertain about how to set the bounds of the input space for performing active learning. If the bounds are too small, we might miss feasible domains; while if the bounds are too large, the active learner has to query more samples than necessary to achieve the required accuracy.

In this paper, we introduce a method of using active learning to expand our knowledge about an unbounded input data space, and discover feasible domains in that space. A na\"{\i}ve solution would be to progressively expand a bounded input space, and apply the existing active learning techniques. However, there are two problems with this na\"{\i}ve solution: (1)~it is difficult to explicitly specify when and how fast we expand the input space; and (2)~the area we need to evaluate increases over time increasing the computational cost. 
Thus existing active learning techniques cannot apply directly to the unbounded domain estimation problem. To the best of our knowledge, \cite{chen2017beyond} is the first to deal with the active learning problem over an unbounded input space (\ie, the unbounded domain estimation problem). The AES method proposed in this paper improves upon that previous work (as illustrated in Sect.~\ref{sec:aes}).

\subsection{Gaussian Process Classification (GPC)}

Gaussian Processes (GP, also called Kriging) are often used as a classifier in active learning~\citep{bryan2006active,lee2008sampling,kapoor2010gaussian,gotovos2013active,chen2014local,chen2015important}. Compared to other commonly used classifiers such as Support Vector Machines or Logistic Regression, GP naturally models probabilistic predictions. This offers us a way to evaluate a sample's informativeness based on its predictive probability distribution.

The Gaussian process uses a kernel (covariance) function $k(\bm{x},\bm{x}')$ to measure the similarity between the two points $\bm{x}$ and $\bm{x}'$. It encodes the assumption that ``similar inputs should have similar outputs". Some commonly used kernels are the Gaussian kernel and the exponential kernel~\citep{krause2007nonmyopic,ma2014active,mac2014hierarchical,kandasamy2017query}. In this paper we use the Gaussian kernel:
\begin{equation}
k(\bm{x},\bm{x}')=\exp\left(-\frac{\|\bm{x}-\bm{x}'\|^2}{2l^2}\right)
\label{eq:gaussian_kernel}
\end{equation}
where $l$ is the length scale.

For binary GP classification, we place a GP prior over the latent function $f(\bm{x})$, and then ``squash" $f(\bm{x})$ through the logistic function to obtain a prior on $\pi(\bm{x}) = \sigma(f(\bm{x})) = P(y=1|\bm{x})$. In the feasible domain identification setting, we can consider $f : \mathcal{X} \in \mathbb{R}^d \rightarrow \mathbb{R}$ as an estimation of feasibility, thus we can call it \textit{estimated feasibility function}. Under the Laplace approximation, given the labeled data $(X_L, \bm{y})$, the posterior of the latent function $f(\bm{x})$ at any $\bm{x}\in X_U$ is a Gaussian distribution: $f(\bm{x})|X_L,\bm{y},\bm{x} \sim \mathcal{N}(\bar{f}(\bm{x}), V(\bm{x}))$ with the mean and the variance expressed as
\begin{align}
\bar{f}(\bm{x}) &= \bm{k}(\bm{x})^TK^{-1}\hat{\bm{f}} = \bm{k}(\bm{x})^T \nabla \log P(\bm{y}|\hat{\bm{f}})
\label{eq:mean}\\
V(\bm{x}) &= k(\bm{x},\bm{x})-\bm{k}(\bm{x})^T(K+W^{-1})^{-1}\bm{k}(\bm{x})
\label{eq:variance}
\end{align}
where $W=-\nabla\nabla\log P(\bm{y}|\bm{f})$ is a diagonal matrix with non-negative diagonal elements; $\bm{f}$ is the vector of latent function values at $X_L$, \ie, $f_i=f(\bm{x}^{(i)})$ where $\bm{x}^{(i)} \in X_L$; $K$ is the covariance matrix of the training samples, \ie, $K_{ij}=k(\bm{x}^{(i)},\bm{x}^{(j)})$; $\bm{k}(\bm{x})$ is the vector of covariances between $\bm{x}$ and the training samples, \ie, $k_i(\bm{x})=k(\bm{x},\bm{x}^{(i)})$; and $\hat{\bm{f}} = \operatornamewithlimits{arg\,max}_{\bm{f}} P(\bm{f}|X,\bm{y})$. When using the Gaussian kernel shown in Eq.~\ref{eq:gaussian_kernel}, $k(\bm{x},\bm{x})=1$. We refer interested readers to a detailed description by Rasmussen~\citep{rasmussen2006gaussian} about the Laplace approximation for the binary GP classifier.

The decision boundary corresponds to $\bar{f}(\bm{x})=0$ or $\bar{\pi}(\bm{x})=0.5$. We predict $y=-1$ when $\bar{f}(\bm{x})<0$, and $y=1$ otherwise.

\section{Active Expansion Sampling (AES)}
\label{sec:aes}

\begin{algorithm*}
  \caption{The Active Expansion Sampling algorithm}
  \begin{algorithmic}[1]
    \Inputs{Query budget $T$ \\
    		Initial point $\bm{x}^{(0)}$ and its label $y_0$ \\
            $d$-dimensional evaluation function $h(\cdot)$ \\
            Hyperparameters $\epsilon$ and $\tau$}
    \Initialize{\strut
      $X_L \gets \{\bm{x}^{(0)}\}$, $Y_L \gets \{y_0\}$, $INIT \gets True$}
    \For{$t = 1,2,\ldots, T$}
      \If{$INIT$ is $True$}
        \If{$X_L$ consists of only one class (all feasible or all infeasible)}
          \State $\bm{c} \gets \bm{x}^{(0)}$
        \Else
          \State $INIT \gets False$
          \State $\bm{c} \gets$ centroid of positive samples in $X_L$
        \EndIf
      \EndIf
      \State Train the GP classifier using $X_L$
      \State Compute $\delta_{exploit}$ using Eq.~\ref{eq:delta_exploit}
      \State $X_U \gets$ uniform samples inside the $(d-1)$-sphere $\mathcal{C}(\bm{x}^{(t-1)}, \delta_{exploit})$
      \State Compute $\bar{f}(\bm{x})$, $V(\bm{x})$, and $p_\epsilon(\bm{x})$ for $\bm{x} \in X_U$ using Eq.~\ref{eq:mean}, (\ref{eq:variance}), and (\ref{eq:p_epsilon})
      \If{there are both $\bar{f}(\bm{x})<0$ and $\bar{f}(\bm{x})>0$ for $\{\bm{x}\in X_U|p_\epsilon(\bm{x})>\tau\}$}
        \Comment Exploitation stage
        \State Select a new query $\bm{x}^{(t)}$ from $X_U$ based on Eq.~\ref{eq:strategy}
      \Else
        \Comment Exploration stage
        \State Compute $\delta_{explore}$ using Eq.~\ref{eq:delta_explore}
        \If{previous iteration is in exploitation stage}
          \State $\hat{\bm{x}} \gets \mathrm{argmax}_{\bm{x}\in X_L} \|\bm{x}-\bm{c}\|$
          \State $X_U \gets$ uniform samples inside the $(d-1)$-sphere $\mathcal{C}(\hat{\bm{x}}, \delta_{explore})$
        \Else
          \State $X_U \gets$ uniform samples inside the $(d-1)$-sphere $\mathcal{C}(\bm{x}^{(t-1)}, \delta_{explore})$
        \EndIf
        \State Compute $\bar{f}(\bm{x})$, $V(\bm{x})$, and $p_\epsilon(\bm{x})$ for $\bm{x} \in X_U$ using Eq.~\ref{eq:mean}, \ref{eq:variance}, and \ref{eq:p_epsilon}
        \State Select a new query $\bm{x}^{(t)}$ from $X_U$ based on Eq.~\ref{eq:strategy}
      \EndIf
      \State $y_t \gets h(\bm{x}^{(t)})$
      \State $X_L \gets X_L\cup\{\bm{x}^{(t)}\}, ~ Y_L \gets Y_L\cup\{y_t\}$
    \EndFor
  \end{algorithmic}
  \label{alg:ade}
\end{algorithm*}

Algorithm~\ref{alg:ade} summarizes our proposed Active Expansion Sampling method. Overall, the method consists of the following steps:
\begin{enumerate}
\item Select an initial sample $\bm{x}^{(0)}$ to label.
\item In each subsequent iteration, 
\begin{enumerate} 
  \item check the exploitation/exploration status (Sect.~\ref{sec:distinguish}),
  \item generate a pool of candidate samples $X_U$ based on the exploitation/exploration status and previous queries (Sect.~\ref{sec:pool_explore} and \ref{sec:pool_exploit}),
  \item train a GP classifier using the labeled set $X_L$ to evaluate the informativeness of candidate samples in $X_U$,
  \item select a sample from $X_U$ based on its informativeness and its distance from $\bm{c}$ (Sect.~\ref{sec:strategy}),
  \item label the new sample and put it into $X_L$.
\end{enumerate}
\item Exit when the query budget is reached.
\end{enumerate}

This AES method improves upon our previous domain expansion method~\citep{chen2017beyond} in several ways. For example, the previous method generates a pool $X_U$ that expands with the explored region each iteration. So its pool size and hence the computational cost increase significantly over time if using a constant sample density. To avoid this problem, this paper proposes a dynamic local pool method (Sect.~\ref{sec:pool}).
Another major difference is that AES provides a verifiable way to distinguish between exploitation and exploration (Sect.~\ref{sec:distinguish}); while the previous method uses a heuristic based on the labels of last few queries (which is more likely to make mistakes). In this section and Sect.~\ref{sec:experiments}, we show comprehensive theoretical analysis and experiments to prove favorable properties of our new method.

\subsection{$\epsilon$-Margin Probability}
\label{sec:strategy}

\begin{figure}
\centering
\includegraphics[width=0.45\textwidth]{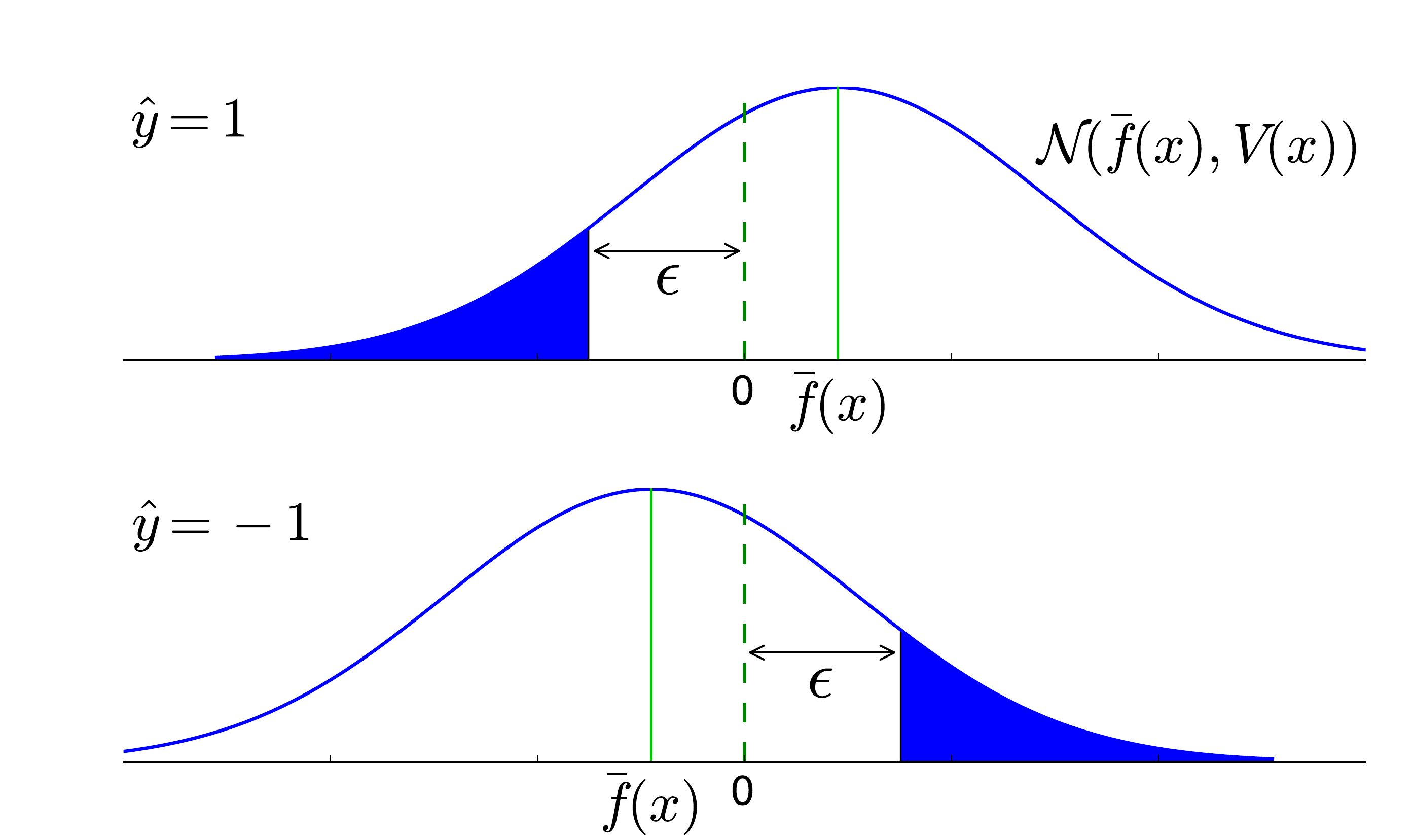}
\caption{The probability density function of the latent function $f(\bm{x})$~\citep{chen2017beyond}. The blue areas represent the $\epsilon$-margin probability $p_\epsilon(\bm{x})$.}
\label{fig:p_epsilon}
\end{figure}

We train a GP classification model to evaluate the informativeness of candidate samples based on the $\epsilon$\textit{-margin probability} (Fig.~\ref{fig:p_epsilon}):
\begin{equation}
\begin{split}
  p_\epsilon(\bm{x}) &= \begin{cases}
    P(f(\bm{x})<-\epsilon|\bm{x}), & \text{if $\hat{y}=1$}\\
    P(f(\bm{x})>\epsilon|\bm{x}), & \text{if $\hat{y}=-1$}
  \end{cases} \\
  &= P(\hat{y}f(\bm{x}) < -\epsilon|\bm{x}) \\
  &= \Phi\left(-\frac{|\bar{f}(\bm{x})|+\epsilon}{\sqrt{V(\bm{x})}}\right)
  \label{eq:p_epsilon}
\end{split}
\end{equation}
where $\hat{y}$ is the estimated label of $\bm{x}$, the margin $\epsilon>0$, and $\Phi(\cdot)$ is the cumulative distribution function of standard Gaussian distribution $\mathcal{N}(0,1)$. The $\epsilon$-margin probability represents the probability of $\bm{x}$ being misclassified with some degree of certainty (controlled by the margin $\epsilon$). 
Let the misclassification loss be
\begin{equation}
\begin{split}
L(\bm{x}) &= \begin{cases}
    \max\{0,-f(\bm{x})\}, & \text{if $y=1$}\\
    \max\{0,f(\bm{x})\}, & \text{if $y=-1$}
  \end{cases}
\end{split}
\label{eq:misclass_loss}
\end{equation}
where $y$ is the true label of $\bm{x}$. $L(\bm{x})$ measures the deviation of the estimated feasibility function value $f(\bm{x})$ from 0 when the class prediction is wrong. Then, based on Eq.~\ref{eq:p_epsilon} and \ref{eq:misclass_loss}, $p_\epsilon(\bm{x})=P(L(\bm{x})>\epsilon)$, which is the probability that the expected misclassification loss exceeds $\epsilon$. A high $p_\epsilon(\bm{x})$ indicates that $\bm{x}$ is very likely to be misclassified, and requires further evaluation. Thus we use this probability to measure informativeness.

\subsection{Exploitation and Exploration}
\label{sec:two_stages}

\begin{figure}
\centering
\includegraphics[width=0.5\textwidth]{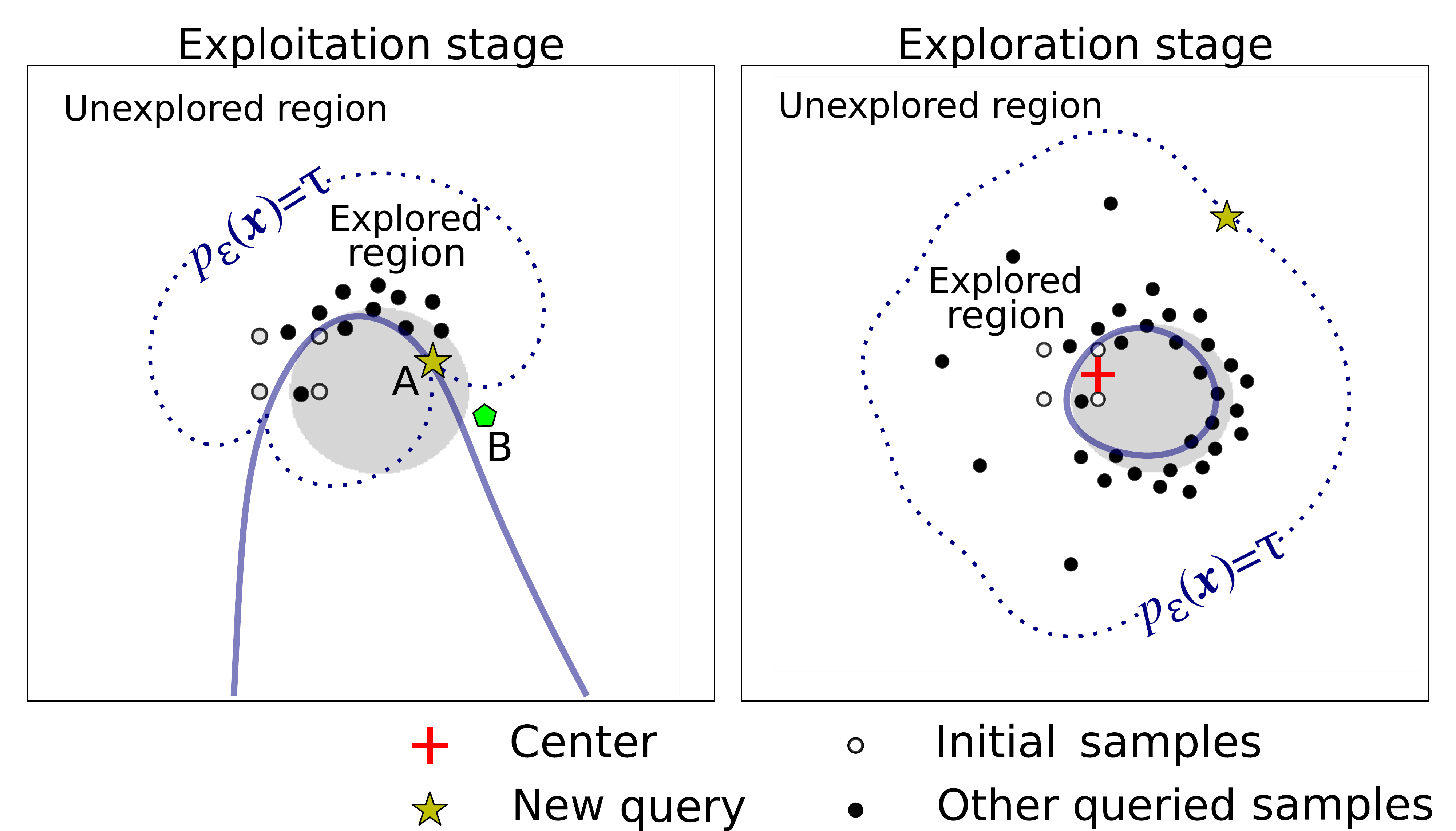}
\caption{Queries at the exploitation stage (left) and the exploration stage (right). The gray area is the ground truth of the feasible domain. The solid line is the decision boundary estimated by the GP classifier; and the dotted line is the isocontour of $p_\epsilon(\bm{x})$. At the exploitation stage (left), the center $\bm{c}$ is the previous query, which makes the next query stay along the decision boundary. At the exploration stage (right), $\bm{c}$ is the centroid of the initial positive samples, which keeps the queries centered around the existing (real-world) samples rather than biasing towards some direction.}
\label{fig:phi_iso}
\end{figure}

Since our input space is unbounded, na\"{i}vely maximizing the $\epsilon$-margin probability (informativeness) will always query points infinitely far away from previous queries.\footnote{A point infinitely far away from previous queries has the $\bar{f}(\bm{x})$ close to 0 and the maximum $V(\bm{x})$, thus the highest $p_\epsilon(\bm{x})$.}
To avoid this issue, one solution is to query informative samples that are close to previously labeled samples. This allows the active learner to progressively expand its knowledge as the queries cover an increasingly large area of the input space. When a new decision boundary is discovered during expansion, we want a query strategy that continues querying points on that decision boundary, such that the new feasible region can be identified as quickly as possible. Therefore, to enable continuous exploitation of the decision boundary, we propose the following query strategy
\begin{equation}
\begin{aligned}
\min_{\bm{x}\in X_U} ~& V(\bm{x}) \\
\text{s.t.} ~& p_\epsilon(\bm{x}) \geq \tau
\end{aligned}
\label{eq:strategy_exploit}
\end{equation}
where $V(\bm{x})$ is the predictive variance at $\bm{x}$, and $\tau$ is a threshold of the informativeness measure $p_\epsilon(\bm{x})$.

\begin{theorem}
The solution to Eq.~\ref{eq:strategy_exploit} will lie at the intersection of the estimated decision boundary ($\bar{f}(\bm{x})=0$) and the isocontour of $p_\epsilon(\bm{x})=\tau$ (Point A in Fig.~\ref{fig:phi_iso}), if that intersection A exists.
\label{thm:intersection}
\end{theorem}

\begin{proof}
In the following proof, we denote $\bar{f}_P=\bar{f}(\bm{x}_P)$, and $V_P=V(\bm{x}_P)$. For a sample $\bm{x}_A$ at the intersection of $\bar{f}(\bm{x})=0$ and $p_\epsilon(\bm{x})=\tau$, we have $\bar{f}_A=0$ and $p_\epsilon(\bm{x}_A)=\Phi(-\epsilon/\sqrt{V_A})=\tau$ (Point A in Fig.~\ref{fig:phi_iso}); and for a sample $\bm{x}_B$ that is any feasible solution to Eq.~\ref{eq:strategy_exploit}, we have $p_\epsilon(\bm{x}_B)=\Phi(-(|\bar{f}_B|+\epsilon)/\sqrt{V_B})\geq\tau$ (Point B in Fig.~\ref{fig:phi_iso}). Thus we get $\epsilon/\sqrt{V_B} \leq (|\bar{f}_B|+\epsilon)/\sqrt{V_B} \leq \epsilon/\sqrt{V_A}$. Therefore, $V_A \leq V_B$. The equality holds when $|\bar{f}_B|=0$ and $p_\epsilon(\bm{x}_B)=\tau$, \ie, $\bm{x}_B$ is also at the intersection of $\bar{f}(\bm{x})=0$ and $p_\epsilon(\bm{x})=\tau$. Thus we proved the intersection has the minimal predictive variance among feasible solutions to Eq.~\ref{eq:strategy_exploit}, and hence it is the optimal solution.
\end{proof}

Theorem~\ref{thm:intersection} indicates that when applying the query strategy shown in Eq.~\ref{eq:strategy_exploit}, the active learner will only query points at the estimated decision boundary\footnote{In Sect.~\ref{sec:aes}, we assume that the queried point is the exact solution to the query strategy. However since we approximate the exact solution by using a pool-based sampling setting, the query may be deviate from the exact solution slightly.} as long as the estimated decision boundary and the isocontour of $p_\epsilon(\bm{x})=\tau$ intersect. The fact that this intersection exists indicates that there are points on the decision boundary that are informative to some extent (\ie, with $p_\epsilon(\bm{x}) \geq \tau$). We call this stage the \textit{exploitation stage}\smoemdash at this stage the active learner exploits the decision boundary. Equation~\ref{eq:strategy_exploit} ensures that the queries are always on the estimated decision boundary until the exploitation stage ends (\ie, there are no longer informative points on the decision boundary).

If the estimated decision boundary and the isocontour of $p_\epsilon(\bm{x})=\tau$ do not intersect, then the algorithm has fully exploited any informative points on the estimated decision boundary (\ie, for all the points on the estimated decision boundary, we have $p_\epsilon(\bm{x})<\tau$). We call this stage the \textit{exploration stage} since the active learner starts to search for another decision boundary (Fig.~\ref{fig:phi_iso}). In this stage, we want the new query to be both informative and close to where we started, since we don't want the new query to deviate too far from where we start. Therefore, the query strategy at the exploration stage is
\begin{equation}
\begin{aligned}
\min_{\bm{x}\in X_U} ~& \|\bm{x}-\bm{c}\| \\
\text{s.t.} ~& p_\epsilon(\bm{x}) \geq \tau
\end{aligned}
\label{eq:strategy_explore}
\end{equation}
where the objective function is the Euclidean distance between $\bm{x}$ and a center $\bm{c}$. This objective keeps the new query selected by Eq.~\ref{eq:strategy_explore} close to $\bm{c}$. In practice, initially when there are only samples from one class, we set $\bm{c}$ as the initial point $\bm{x}^{(0)}$ to keep new queries close to where we start; once there are both positive and negative samples, we set $\bm{c}$ as the centroid of these initial positive samples, since we want to keep new queries close to the initial feasible region.

\begin{theorem}
Given $\bm{x}^*$ as the solution to Eq.~\ref{eq:strategy_explore}, we have $p_\epsilon(\bm{x}^*)=\tau$, if $p_\epsilon(\bm{c})<\tau$.
\label{thm:isocontour}
\end{theorem}

\begin{proof}
Since $p_\epsilon(\bm{c})<\tau$, $\bm{c}$ itself is not the solution of Eq.~\ref{eq:strategy}. Thus $\|\bm{x}^*-\bm{c}\|>0$. Then we have $p_\epsilon(\bm{x})<\tau$ at any point within a $(d-1)$-sphere centered at $\bm{c}$ with radius $\|\bm{x}^*-\bm{c}\|$, because otherwise the query will be inside the sphere. Thus on that sphere we have $p_\epsilon(\bm{x}) \leq \tau$. So $p_\epsilon(\bm{x}^*) \leq \tau$, since $\bm{x}^*$ is on that sphere. Because $\bm{x}^*$ is a feasible solution to Eq.~\ref{eq:strategy}, we also have $p_\epsilon(\bm{x}^*) \geq \tau$ at $\bm{x}^*$. Therefore $p_\epsilon(\bm{x}^*)=\tau$.
\end{proof}

Theorem~\ref{thm:isocontour} shows that in each iteration, the optimal query $\bm{x}^*$ selected by Eq.~\ref{eq:strategy_explore} is on the isocontour of $p_\epsilon(\bm{x})=\tau$.

For both Eq.~\ref{eq:strategy_exploit} and \ref{eq:strategy_explore}, the feasible solutions are in the region of $p_\epsilon(\bm{x})\geq\tau$. Intuitively this means that we only query samples with at least some level of informativeness. We call the region where $p_\epsilon(\bm{x})\geq\tau$ the \textit{unexplored region}, since it contains informative samples (feasible solutions) that our query strategy cares about; while we call the rest of the input space ($p_\epsilon(\bm{x})\leq\tau$) the \textit{explored region} (Fig.~\ref{fig:phi_iso}).

The upper bound of $p_\epsilon(\bm{x})$ is $\Phi(-\epsilon/\sup_{\bm{x}} V(\bm{x}))$, and it lies infinitely far away from the labeled samples. In Eq.~\ref{eq:variance}, $K+W^{-1}$ is positive semidefinite, thus ${\bm{k}(\bm{x})}^T(K+W^{-1})^{-1}\bm{k}(\bm{x}) \geq 0$ and $V(\bm{x})\leq k(\bm{x},\bm{x})$. For a kernel $k(\cdot)$ with $k(\bm{x},\bm{x})=1$ (\eg, the Gaussian or the exponential kernel), we have $V(\bm{x})\leq 1$. Thus $p_\epsilon(\bm{x}) \leq \Phi(-\epsilon)$. To ensure that Eq.~\ref{eq:strategy} has a feasible solution, we have to set $\tau \leq \Phi(-\epsilon)$. Therefore, we can set $\tau=\Phi(-\eta\epsilon)$, where $\eta \geq 1.0$. Then the constraint in Eq.~\ref{eq:strategy_exploit} and \ref{eq:strategy_explore} can be expressed as
\begin{equation*}
\Phi\left(-\frac{|\bar{f}(\bm{x})|+\epsilon}{\sqrt{V(\bm{x})}}\right) \geq \Phi(-\eta\epsilon)
\end{equation*}
which can be written as
\begin{equation}
\eta\epsilon\sqrt{V(\bm{x})} - |\bar{f}(\bm{x})| \geq \epsilon
\label{eq:straddle0}
\end{equation}
The left-hand side of Eq.~\ref{eq:straddle0} is identical to the acquisition function of the \textit{straddle heuristic} when $\eta\epsilon=1.96$~\citep{bryan2006active}. The straddle heuristic queries the sample with the largest value of the acquisition function. This acquisition function accounts for the \textit{ambiguity} of samples in terms of their confidence intervals~\citep{gotovos2013active}:
\begin{equation*}
\begin{split}
a(\bm{x}) &= \min\{-\min Q(\bm{x}), \max Q(\bm{x})\} \\
&= 1.96\sqrt{V(\bm{x})} - |\bar{f}(\bm{x})|
\end{split}
\end{equation*}
where $Q(\bm{x})$ is the 95\% confidence interval of $\bm{x}$. 

Substituting Eq.~\ref{eq:straddle0} for the constraint in Eq.~\ref{eq:strategy_exploit} and \ref{eq:strategy_explore}, and combining the exploitation and exploration stages, our overall query strategy becomes
\begin{equation}
\begin{aligned}
\min_{\bm{x}\in X_U} ~& V(\bm{x})^\alpha\|\bm{x}-\bm{c}\|^{1-\alpha} \\
\text{s.t.} ~& \eta\epsilon\sqrt{V(\bm{x})} - |\bar{f}(\bm{x})| \geq \epsilon
\end{aligned}
\label{eq:strategy}
\end{equation}
where the indicator $\alpha$ is 1 at the exploitation stage, and 0 otherwise. Section~\ref{sec:distinguish} introduces how to set $\alpha$ (\ie, when to exploit vs explore).

In general, the unbounded domain estimation problem can be solved using a family of query strategies with the following form
\begin{equation*}
\begin{aligned}
\min_{\bm{x}\in X_U} ~& D(\bm{x}) \\
\text{s.t.} ~& I(\bm{x}) \geq \tau
\end{aligned}
\end{equation*}
where $D(\bm{x})$ is a function that increases as $\bm{x}$ moves away from the labeled samples, and $I(\bm{x})$ is the informativeness measure that is used in any bounded active learning methods. Our query strategies of Eqn~\ref{eq:strategy_exploit} and \ref{eq:strategy_explore} all have this form. Comparatively, for bounded active learning methods, the query strategies are usually in the form of $\max_{\bm{x}\in X_U} I(\bm{x})$.

\section{Dynamic Local Pool Generation}
\label{sec:pool}

We cast our problem as pool-based sampling by generating a pool of unlabeled instances de novo in each iteration. A na\"{i}ve way to generate this pool is to try to sample points anywhere near the $p_\epsilon(\bm{x})=\tau$ isocontour. However, intuitively, as the algorithm searches progressively larger volumes of the input space, the pool volume will likewise expand. This expansion means that the size of the pool will increase dramatically over time (assuming we want a constant sample density). This increase, however, makes the computation of Eq.~\ref{eq:mean} and \ref{eq:variance} expensive during later expansion stages.

To bypass this problem, we propose a \textit{dynamic local pool} method that generates the pool of candidate samples only at a certain location in each iteration, rather than sampling the entire domain.\footnote{Sampling methods like random sampling or Poisson-disc sampling~\citep{bridson2007fast} can be used to generate the pool. We use random sampling here thereby for simplicity. The specific choice of the sampling method within the local pool is not central to the overall method.} The key insight behind our local pooling method is that while the optimal solution to Eq.~\ref{eq:strategy} can, in principle, occur anywhere on the $p_\epsilon(\bm{x})=\tau$ isocontour, in practice, multiple points on the isocontour are equally optimal. All we need to do is sample points around any one of those optima. Below, we derive guarantees for how to sample volumes near one of those optima, thus only needing to sample a small fraction of the total domain volume.

\subsection{Scope of an Optimal Query}

\begin{theorem}
Let $\delta$ be the distance between an optimal query\footnote{The optimal query means the exact solution to the AES query strategy shown in Eq.~\ref{eq:strategy_exploit}, \ref{eq:strategy_explore}, or \ref{eq:strategy}.} and its nearest labeled sample.  We have
\begin{equation}
\delta < \beta l
\label{eq:delta}
\end{equation}
where $\beta$ is a coefficient depends on $\epsilon$, $\eta$, and the GP model.
\label{thm:delta}
\end{theorem}

We include the proof of Theorem~\ref{thm:delta} and the way of computing $\beta$ in the appendix (Sect.~\ref{pf:delta}).
Theorem~\ref{thm:delta} indicates that if we set the pool boundary by extending the current labeled sample range by $\beta l$, then that pool is guaranteed to contain all solutions to Eq.~\ref{eq:strategy}; that is, extending the overall pool boundary further will not increase the chances of sampling near $p_\epsilon(\bm{x})=\tau$, and will only decrease the sample density (given a fixed pool size) or increase the evaluated samples (given a fixed sample density). However, if we generate the pool based solely on this principle (\ie, extending the current labeled sample range by $\beta l$), the pool size will still increase over time as the domain size grows. The next two sections show how, for the exploration and exploitation stages respectively, we can further reduce the sample boundary to only a local hyper-sphere.

\subsection{Pool for the Exploration Stage}
\label{sec:pool_explore}

\begin{figure}
\centering
\includegraphics[width=0.5\textwidth]{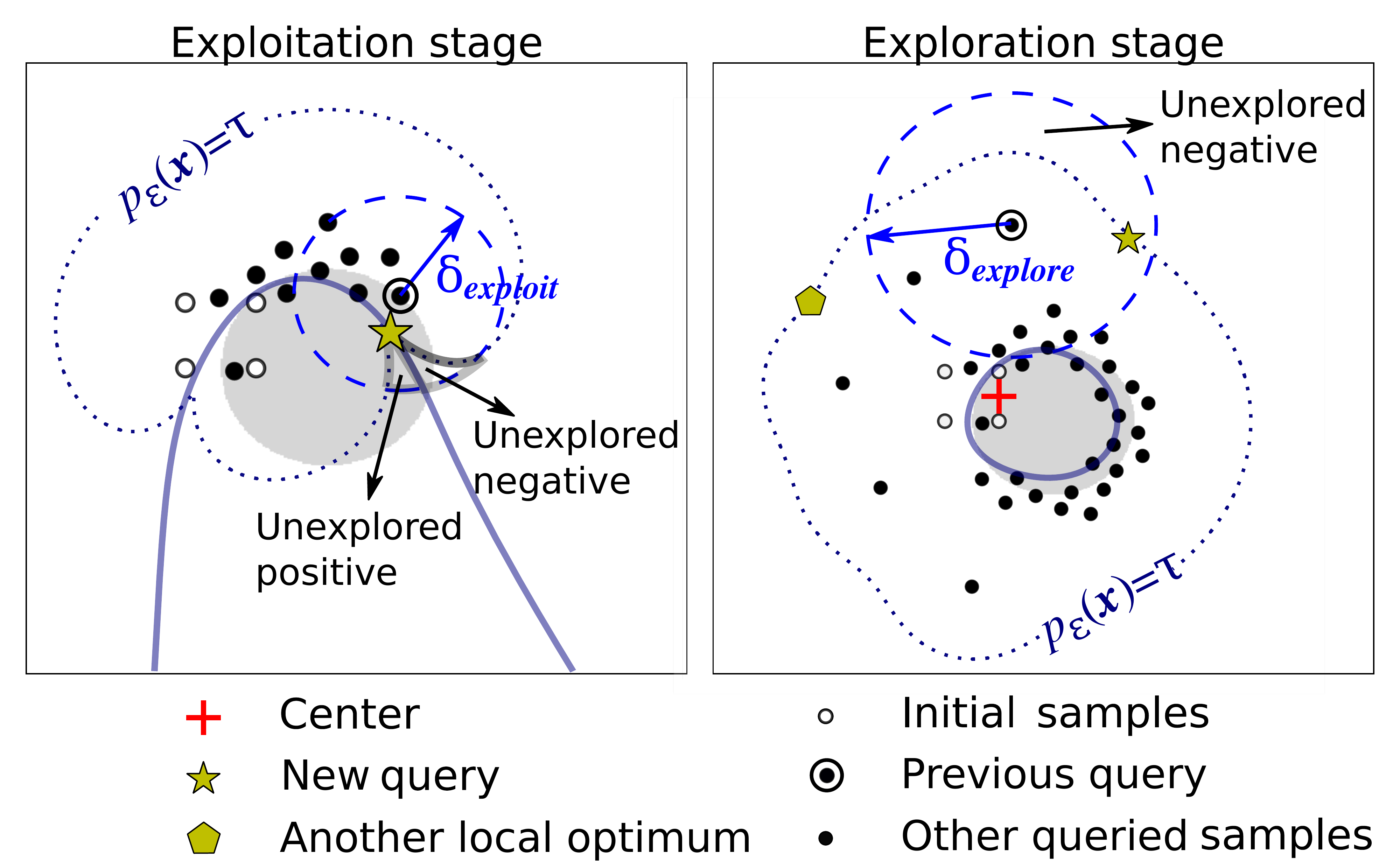}
\caption{Dynamic local pools (dashed circles) at the exploitation stage (left) and the exploration stage (right). During the exploitation stage, the estimated decision boundary divided the unexplored region into two subregions: unexplored negative $\mathcal{R}_1 = \{\bm{x}|p_\epsilon(\bm{x})>\tau, \bar{f}(\bm{x})<0$\} and unexplored positive $\mathcal{R}_2 = \{\bm{x}|p_\epsilon(\bm{x})>\tau, \bar{f}(\bm{x})>0$\}; while during the exploration stage, there will be at most one of $\mathcal{R}_1$ and $\mathcal{R}_2$ in the unexplored region. This property can be used to distinguish between the exploitation/exploration stages.}
\label{fig:pool}
\end{figure}

\begin{theorem}
During the exploration stage of Active Expansion Sampling, the distance between an optimal query and its nearest labeled sample is
\begin{equation}
\delta < \delta_{explore} = \beta l
\label{eq:delta_explore}
\end{equation}
\label{thm:delta_explore}
\end{theorem}

Theorem~\ref{thm:delta_explore} is derived from Eq.~\ref{eq:delta}.
The nearest labeled sample of the optimal query could be any border point (a sample lying on the periphery of the labeled set). There are multiple local optima that are equally useful for expanding the explored region (Fig.~\ref{fig:pool}). Thus we just sample near one of those optima. Specifically, we approximate the nearest labeled sample as the previous query. With this approximation, incorporating Theorem~\ref{thm:delta_explore}, the optimal query will be inside $\mathcal{C}(\bm{x}^{(t-1)}, \delta_{explore})$, the $(d-1)$-sphere with a radius of $\delta_{explore}$, centered at the previous query $\bm{x}^{(t-1)}$. Thus during the exploration stage, we set the pool boundary to be that sphere (Fig.~\ref{fig:pool}).

Sometimes when AES switches from exploitation to exploration, the previous query may not lie on the periphery of the labeled samples. %
This causes samples around the previous query to have low values of $p_\epsilon(\bm{x})$. In this case, there might not be feasible solution to Eq.~\ref{eq:strategy}. Thus, every time AES switches from exploitation to exploration, we center the pool around the farthest labeled sample from the centroid of the initial positive samples (\ie, $\mathrm{argmax}_{\bm{x}\in X_L} \|\bm{x}-\bm{c}\|$). This ensures that AES generates pool samples near the periphery of the labeled samples.

\subsection{Pool for the Exploitation Stage}
\label{sec:pool_exploit}

\begin{theorem}
During the exploitation stage of Active Expansion Sampling, the distance between an optimal query and its nearest labeled sample is
\begin{equation}
\delta < \delta_{exploit} = \gamma l
\label{eq:delta_exploit}
\end{equation}
where $\gamma$ is a coefficient depends on $\epsilon$, $\eta$, and the GP model.
\label{thm:delta_exploit}
\end{theorem}

We include the proof of Theorem~\ref{thm:delta_exploit} and the way of computing $\gamma$ in the appendix (Sect.~\ref{pf:delta_exploit}).
Similar to the exploration stage, based on Theorem~\ref{thm:delta_exploit}, we define the pool boundary during the exploitation stage as $\mathcal{C}(\bm{x}^{(t-1)}, \delta_{exploit})$, a $(d-1)$-sphere with a radius of $\delta_{exploit}$, centered at the previous query $\bm{x}^{(t-1)}$ (Fig.~\ref{fig:pool}).

\subsection{Choosing when to Exploit versus Explore}
\label{sec:distinguish}

Since we use different rules to generate the pool at the exploitation and exploration stage, we need to distinguish between the two stages at the beginning of each iteration. In the exploitation stage, according to Theorem~\ref{thm:delta_exploit}, the optimal query lies within the $(d-1)$-sphere $\mathcal{C}(\bm{x}^{(t-1)}, \delta_{exploit})$ centered at the previous query. While, according to Theorem~\ref{thm:intersection}, that same query must lie where the estimated decision boundary and the isocontour of $p_\epsilon(\bm{x})=\tau$ intersect.
Thus, the decision boundary and the isocontour divide the sphere $\mathcal{C}$ into four regions (Fig.~\ref{fig:pool}):\\
\textit{unexplored negative} $\mathcal{R}_1 = \{\bm{x}|p_\epsilon(\bm{x})>\tau, \bar{f}(\bm{x})<0$\}; 
\\\textit{unexplored positive} $\mathcal{R}_2 = \{\bm{x}|p_\epsilon(\bm{x})>\tau, \bar{f}(\bm{x})>0$\}; 
\\\textit{explored negative} $\mathcal{R}_3 = \{\bm{x}|p_\epsilon(\bm{x})<\tau, \bar{f}(\bm{x})<0$\}; and 
\\\textit{explored positive} $\mathcal{R}_4 = \{\bm{x}|p_\epsilon(\bm{x})<\tau, \bar{f}(\bm{x})>0$\}.\\
In contrast, during exploration the estimated decision boundary and the $p_\epsilon(\bm{x})=\tau$ isocontour do not intersect\smoemdash meaning, unlike exploitation, there exist only two of the four regions (either $\mathcal{R}_1$ \& $\mathcal{R}_3$ \textit{or} $\mathcal{R}_2$ \& $\mathcal{R}_4$). In particular, within the unexplored region, $\bar{f}(\bm{x})$ will be either all positive or all negative, \ie, $\mathcal{R}_1$ and $\mathcal{R}_2$ cannot exist simultaneously (Fig.~\ref{fig:pool}). 

We use this property to detect exploitation or exploration by generating a pool (a set of uniformly distributed samples) within the boundary $\mathcal{C}(\bm{x}^{(t-1)}, \delta_{exploit})$ and checking if, for samples with $p_\epsilon(\bm{x})>\tau$, samples differ in $\bar{f}(\bm{x})>0$ and $\bar{f}(\bm{x})<0$. If so, AES is in the exploitation stage; otherwise it is in the exploration stage.

\section{Theoretical Analysis}

In this section, we derive a theoretical accuracy bound for AES with respect to its hyperparameters. We further discuss the influence of those hyperparameters on the classification accuracy, the query density, and the exploration speed. The results of this section guides the selection of proper hyperparameters given an accuracy or budget requirement.

\subsection{Accuracy Analysis}
\label{sec:acc_bound}

It is impossible to discuss the function accuracy across the entire input space, since the input space is unbounded. However, we can consider ways to bound the accuracy within bounded explored regions at any time step.

As mentioned in Sect.~\ref{sec:strategy}, $p_\epsilon(\bm{x})=P(L(\bm{x})>\epsilon)$, where $L(\bm{x})$ is the misclassification loss at $\bm{x}$ defined in Eq.~\ref{eq:misclass_loss}. Thus within the explored region, we have
\begin{equation*}
P(L(\bm{x})\geq\epsilon)\leq\tau ~~~ \forall \bm{x}\in\{\bm{x}|p_\epsilon(\bm{x})\leq\tau\}
\end{equation*}
or
\begin{equation}
P(L(\bm{x})\leq\epsilon)\geq 1-\tau ~~~ \forall \bm{x}\in\{\bm{x}|p_\epsilon(\bm{x})\leq\tau\}
\label{eq:accuracy}
\end{equation}

This shows that at any location within the explored region of the input space, the proposed method guarantees an upper bound $\epsilon$ of misclassification loss with a probability of at least $1-\tau$ at any given point. Since, in the exploration stage, the estimated decision boundary lies inside the $p_\epsilon(\bm{x})\leq\tau$ region (as discussed in Sect.~\ref{sec:two_stages}), we have 
\begin{equation*}
P(L(\bm{x})\leq\epsilon)\geq 1-\tau ~~~ \forall \bm{x}\in\{\bm{x}|\bar{f}(\bm{x})=0\}
\end{equation*}
This means that in the exploration stage, the estimated decision boundary $\bar{f}(\bm{x})=0$ lies in between the isocontours of $f(\bm{x})=\pm\epsilon$ with a probability of at least $1-\tau$, where $f$ is the true latent function.

Note that Eq.~\ref{eq:accuracy} shows that AES's accuracy bound within the explored region is independent of the number of iterations or labeled samples. One advantage of keeping a constant accuracy bound for AES is that the accuracy in the explored region meets our requirements\footnote{We can set $\epsilon$ and $\tau$ such that the accuracy bound is as required. Details about how to set hyperparameters are in Sect.~\ref{sec:hyperparam}.} whenever AES stops. This also means that the estimation within the explored region is reliable at any iteration (although this is not true if one includes the unexplored region). In contrast, bounded active learning methods usually only achieve required accuracy after a certain number of iterations, before which the estimation may not be reliable. Therefore, AES can be used for real-time prediction of samples' feasibility in the explored region.

\subsection{Query Density}

In Gaussian Processes, given a fixed homoscedastic Gaussian or exponential kernel, we can measure the query density by looking at the predictive variance at queried points. According to Eq.~\ref{eq:variance}, $V(\bm{x})$ only depends on $\bm{k}(\bm{x})$, which is affected by the distances between $\bm{x}$ and other queries. A smaller variance at a query indicates that it is closer to other queries, and hence a higher query density; and vise versa.

\begin{theorem}
The predictive variance of an optimal query in the exploitation and exploration stage is
\begin{equation}
V(\bm{x}_{exploit}) = \frac{1}{\eta^2}
\label{eq:v_exploit}
\end{equation}
and
\begin{equation}
V(\bm{x}_{explore}) = \frac{1}{\eta^2}\left(1+\frac{|\bar{f}(\bm{x}_{explore})|}{\epsilon}\right)^2
\label{eq:v_explore}
\end{equation}
where $\bm{x}_{exploit}$ and $\bm{x}_{explore}$ are optimal queries at the exploitation stage and exploration stage, respectively.
\label{thm:density}
\end{theorem}

The proof of Theorem~\ref{thm:density} is in the appendix (Sect.~\ref{pf:density}). This theorem indicates that the predictive variances of queries at the exploitation stage are always smaller than those at the exploration stage (as $|\bar{f}(\bm{x}_{explore})|>0$). Thus the query density at the exploitation stage is always higher than that at the exploration stage. The property of having a denser set of points along the decision boundary (queried during the exploitation stage) and a sparser set of points at other regions (queried during the exploration stage) is desirable because we want to save our query budget for refining the decision boundary rather than other regions of the input space.

Equation~\ref{eq:accuracy} and \ref{eq:v_exploit} also reflect the trade-off between the accuracy and the running time. When the query density near the decision boundary is high (small $V(\bm{x}_{exploit})$ in Eq.~\ref{eq:v_exploit}), $\eta$ is large, thus $\tau$ in Eq.~\ref{eq:accuracy} is small, which means our model will have a higher probability of having a misclassification loss less than $\epsilon$. However, as the query density gets higher, we need more queries to cover a certain region, thus the running time increases.

\subsection{Influence of Hyperparameters}
\label{sec:hyperparam}

There are four hyperparameters that control Active Expansion Sampling\smoemdash the initial point $\bm{x}^{(0)}$, $\epsilon$ and $\eta$ in the exploitation/exploration stage, and the length scale $l$ of the GP kernel.
The choice of the kernel function and length scale depends on assumptions regarding the nature and smoothness of the underlying feasibility function. Such kernel choices have been covered extensively in prior research and we refer interested readers to ~\citep{rasmussen2006gaussian} for multiple methods of choosing $l$. Note that it is difficult to optimize the length scale at each iteration, since the length scale will eventually be pushed to extremes. In the exploitation stage, for example, once the length scale is smaller than the previous iteration, the distance between the new query and its nearest query will also be smaller (due to Eq.~12). Then the maximum marginal likelihood estimation will result in a smaller length scale, as the estimated function is steeper. This process will repeat and eventually cause the optimal length scale to converge to 0. The initial point $\bm{x}^{(0)}$ can be any point not too far away from the boundary of feasible regions, since otherwise it will take a large budget to just search for a sample from the opposite class. Here we focus on the analysis of the other two hyperparameters\smoemdash $\epsilon$ and $\eta$.

According to Eq.~\ref{eq:accuracy}, $\epsilon$ and $\tau$ affect the classification accuracy in a probabilistic way. When $\tau=\Phi(-\eta\epsilon)$, we have $P(L(\bm{x})\leq\epsilon)\geq 1-\Phi(-\eta\epsilon)$ in the explored region. This offers us a guideline for setting $\epsilon$ and $\eta$ with respect to a given accuracy requirement.

According to Eq.~\ref{eq:v_exploit} and \ref{eq:v_explore}, $\eta$ controls the density of queries in both exploitation and exploration stages. Specifically, as we increase $\eta$, $V_{exploit}$ and $V_{explore}$ decreases, increasing the query density and essentially placing labeled points closer together.

In contrast, $\epsilon$ only controls the distances between queries in the exploration stage.
\footnote{Technically, due to sampling error introduced when generating the pool, the exploitation stage will be influenced by $\epsilon$ (since $\bar{f}(\bm{x}^*)$ is only $\approx 0$). But this effect is negligible compared to $\epsilon$'s influence on the exploration stage.} Increasing $\epsilon$ decreases $V_{explore}$ and hence increases the density of queries in the exploration stage. 
This density of queries affects (1) how fast we can expand the explored region, and (2) how likely we are to capture small feasible regions. When $\eta$ or $\epsilon$ increases, we expand the explored region slower, making it more likely that we will discover smaller feasible regions. Likewise, we also slow down the expansion in exploitation stages, making the classifier more likely to capture a sudden change along domain boundaries.

Note that when $\epsilon=0$, the constraint of $p_\epsilon(\bm{x})\geq\tau$ in Eq.~\ref{eq:strategy} is equivalent to $\bar{f}(\bm{x})=0$, thus theoretically all queries should lie near the estimated decision boundary. In this case, the Active Expansion Sampling acts like \textit{Uncertainty Sampling}~\citep{lewis1994heterogeneous,lewis1994sequential}. In practice, however, AES will be unable to find a feasible solution when $\epsilon=0$ since no candidate sample will be exactly on the decision boundary under the pool-based sampling setting.



\section{Experimental Evaluation}
\label{sec:experiments}

We evaluate the performance of AES in capturing feasible domains using both synthesized and real-world examples. The performance is measured by the F1 score, which is expressed as
\begin{equation*}
F_1 = 2\cdot\frac{\text{precision}\cdot\text{recall}}{\text{precision}+\text{recall}}
\end{equation*}
where
\begin{equation*}
\text{precision}=\frac{\text{true positives}}{\text{true positives}+\text{false positives}}
\end{equation*}
and
\begin{equation*}
\text{recall}=\frac{\text{true positives}}{\text{true positives}+\text{false negatives}}
\end{equation*}

We compare AES with two conventional bounded adaptive sampling methods\smoemdash the Neighborhood-Voronoi (NV) algorithm~\citep{singh2017sequential} and the straddle heuristic~\citep{bryan2006active}. We also investigate the effects of noise and dimensionality on AES.

We use the same pool size (500 candidate samples\footnote{For NV algorithm, its pool size refers to the test samples generated for the Monte Carlo simulation.}) in all the experiments. In Fig.~\ref{fig:hyppara}-\ref{fig:f1s_noisy} and \ref{fig:highdim}, the F1 scores are averaged over 100 runs. We run all 2-dimensional experiments on a Dell Precision Tower 5810 with 16 GB RAM, a 3.5 GHz Intel Xeon CPU E5-1620 v3 processor, and a Ubuntu 16.04 operating system. We run all higher-dimensional experiments on a Dell Precision Tower 7810 with 32 GB RAM, a 2.4 GHz Intel Xeon CPU E5-2620 v3 processor, and a Red Hat Enterprise Linux Workstation 7.2 operating system. The Python code needed to reproduce our AES algorithm, our baseline implementations of NV and Straddle, and all of our below experiments is available at \url{https://github.com/IDEALLab/Active-Expansion-Sampling}.


\begin{figure*}
\centering
\includegraphics[width=1\textwidth]{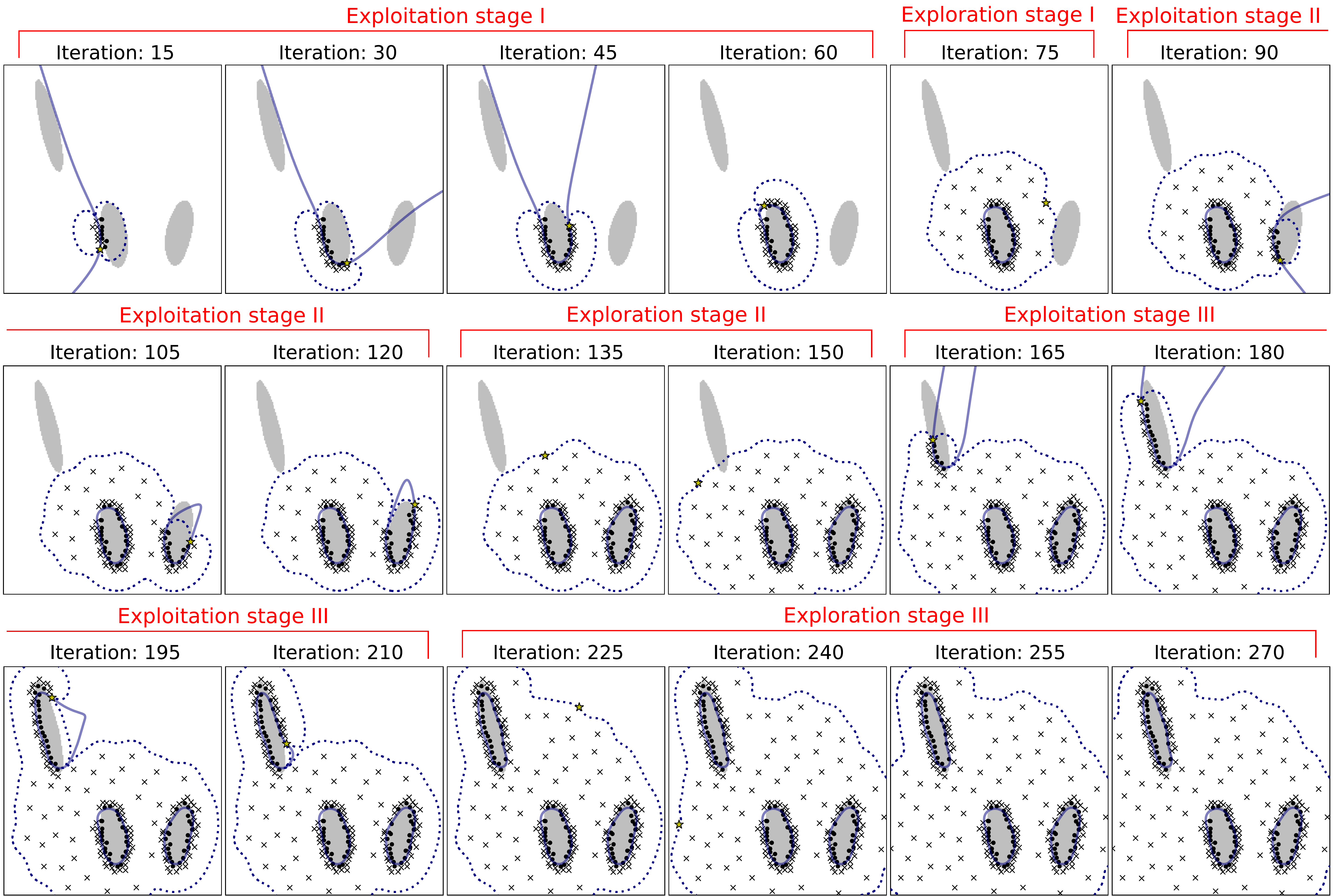}
\caption{Querying sequence for Active Expansion Sampling ($\epsilon=0.5$ and $\eta=1.3$). The solid lines are estimated decision boundaries, and the dotted lines are the isocontour $p_\epsilon(\bm{x})=\tau$. The gray areas are actual feasible regions.}
\label{fig:opt}
\end{figure*}

\begin{figure*}
\centering
\subfloat[Neighborhood-Voronoi algorithm.]{
\includegraphics[width=0.51\textwidth]{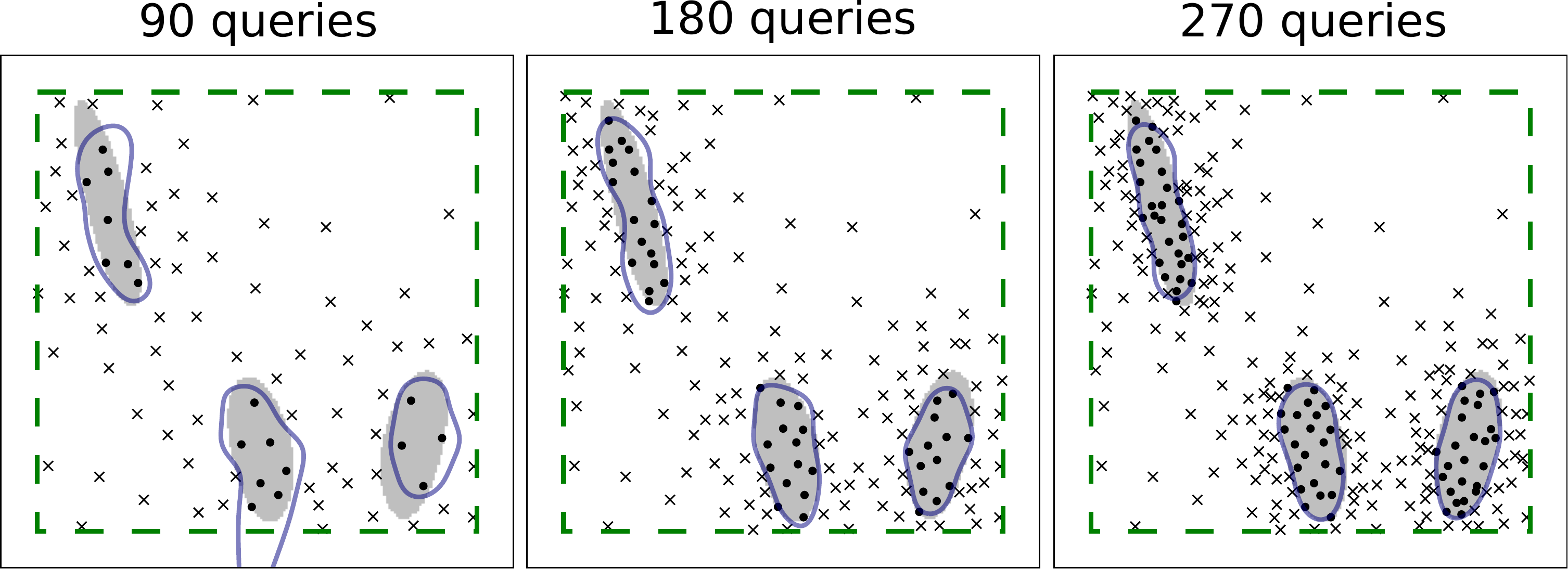}
\label{fig:opt_nv}}
\subfloat[Straddle heuristic.]{
\includegraphics[width=0.49\textwidth]{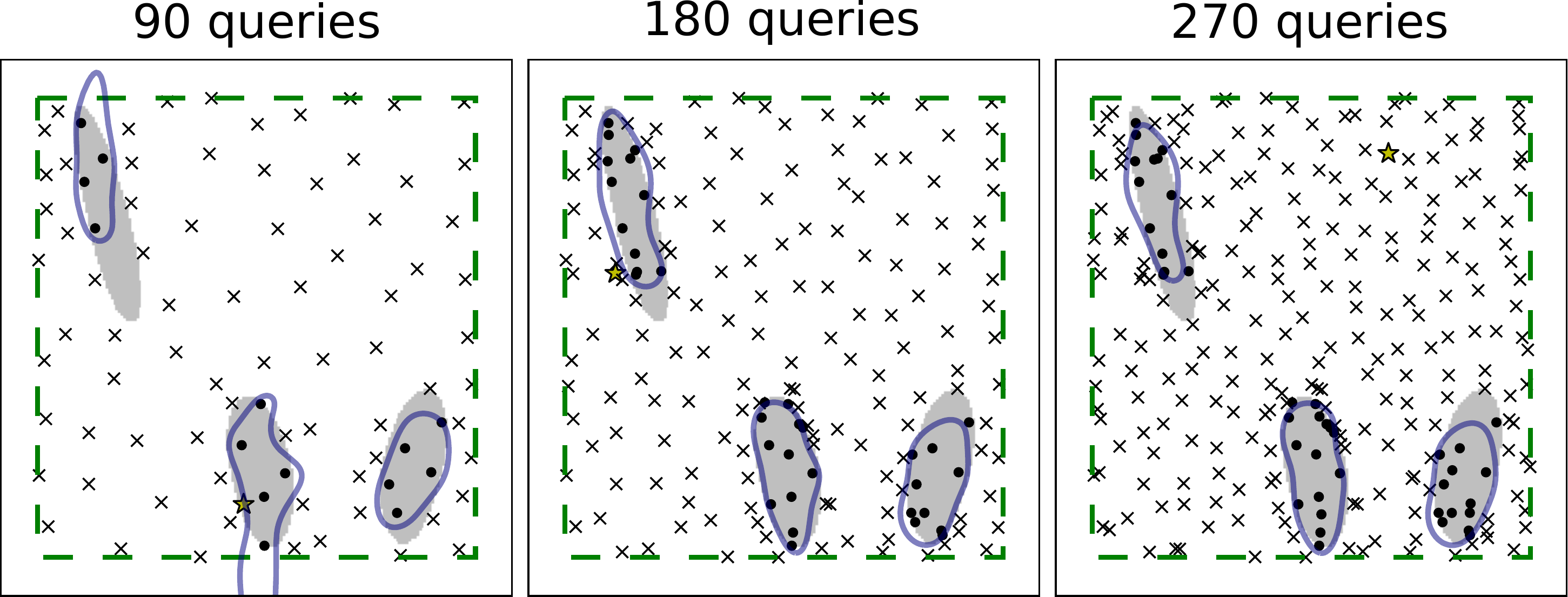}
\label{fig:opt_straddle}}
\caption{Querying sequence for bounded adaptive sampling methods. The dashed lines are pool boundaries.}
\label{fig:opt_bounded}
\end{figure*}

\begin{figure}
\centering
\includegraphics[width=0.5\textwidth]{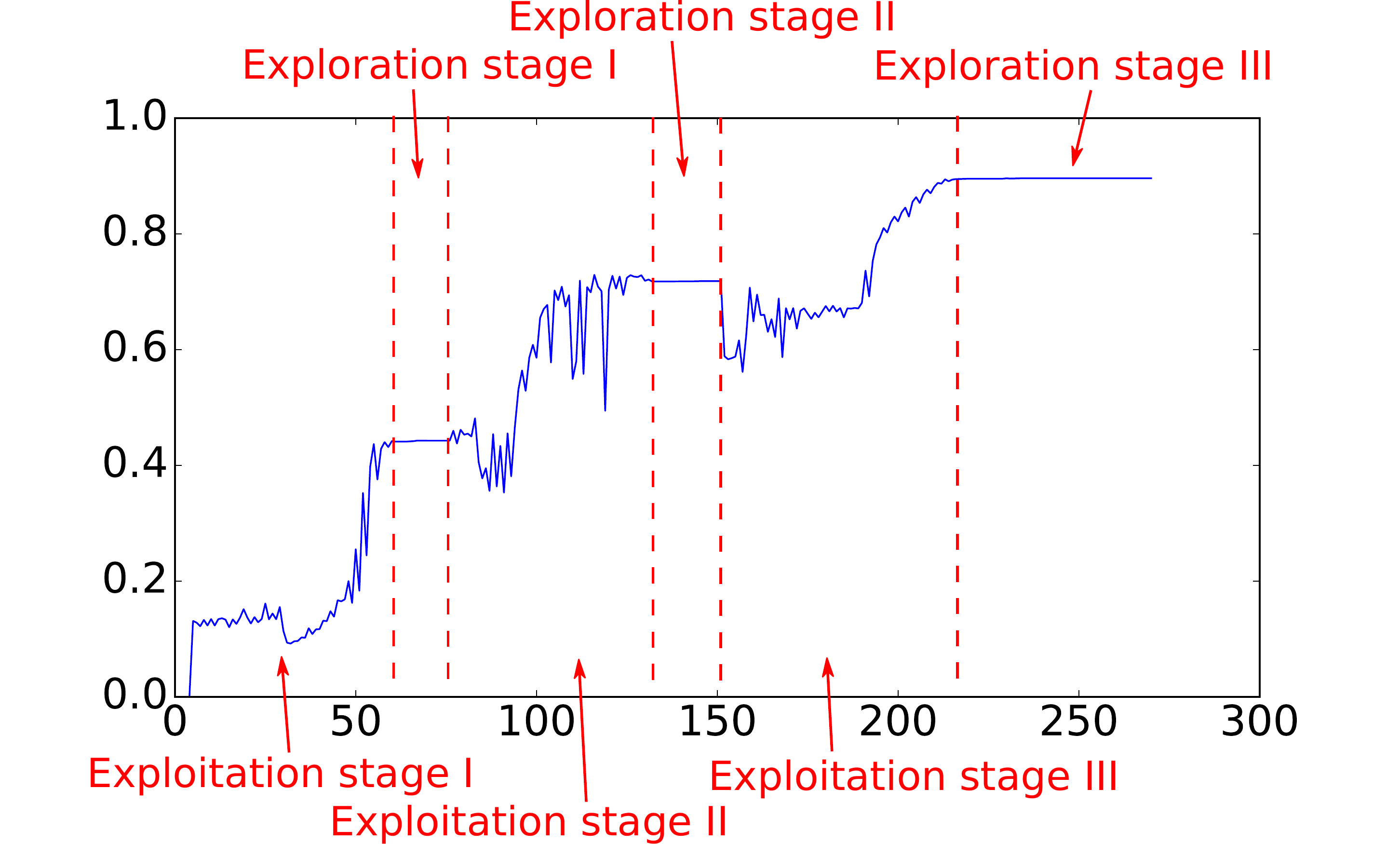}
\caption{F1 score plot for Fig.~\ref{fig:opt}. During exploitation stages, the F1 score increases stochastically as the decision boundary changes; while in the exploration stage, the current decision boundaries have been exploited and do not change, thus the F1 score also does not change.}
\label{fig:opt_f1s}
\end{figure}

\begin{figure*}
\centering
\subfloat[Changing $\epsilon$ ($\eta=1.3$).]{
\includegraphics[width=0.5\textwidth]{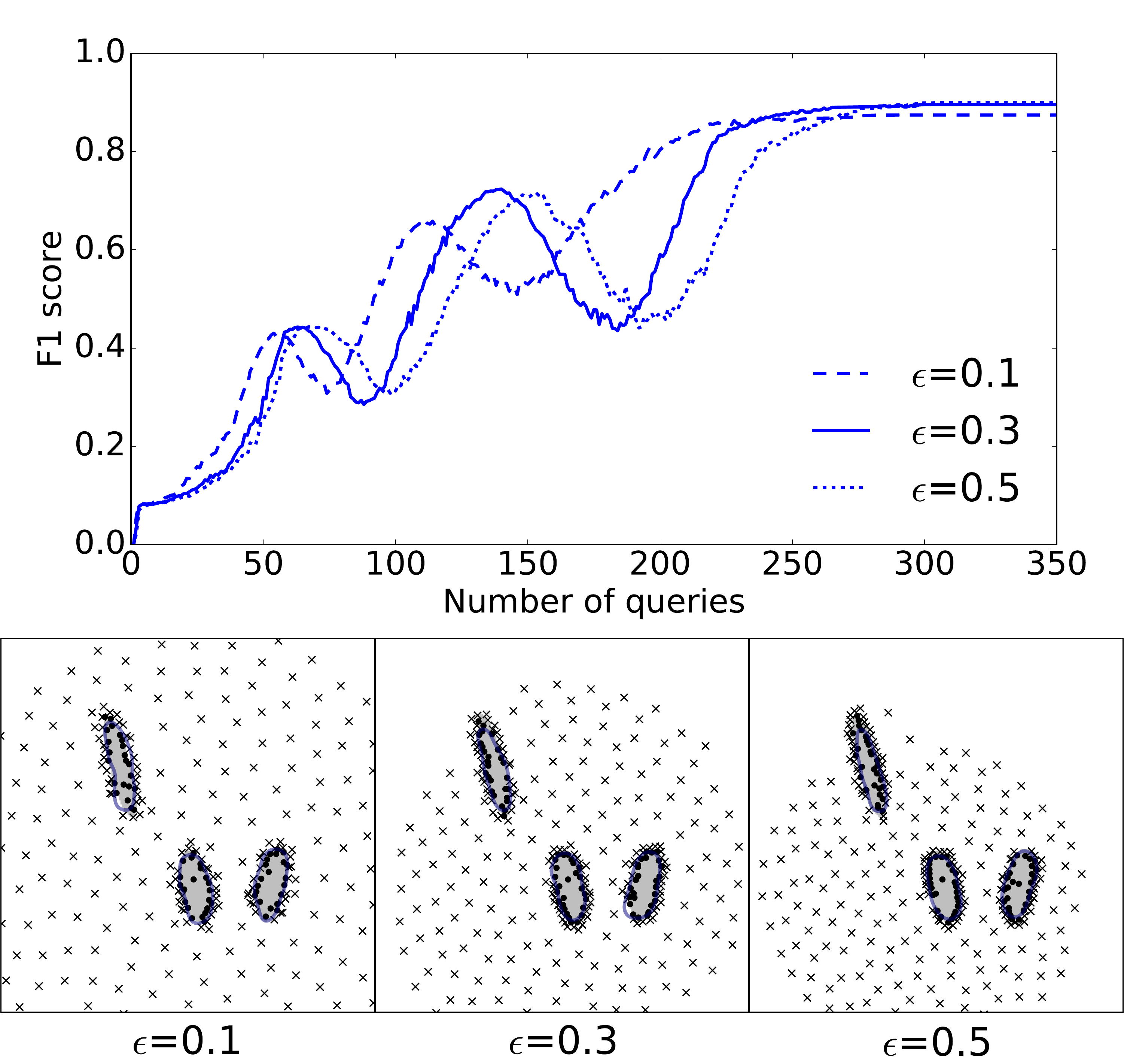}
\label{fig:f1s_epsilon}}
\subfloat[Changing $\eta$ ($\epsilon=0.3$).]{
\includegraphics[width=0.5\textwidth]{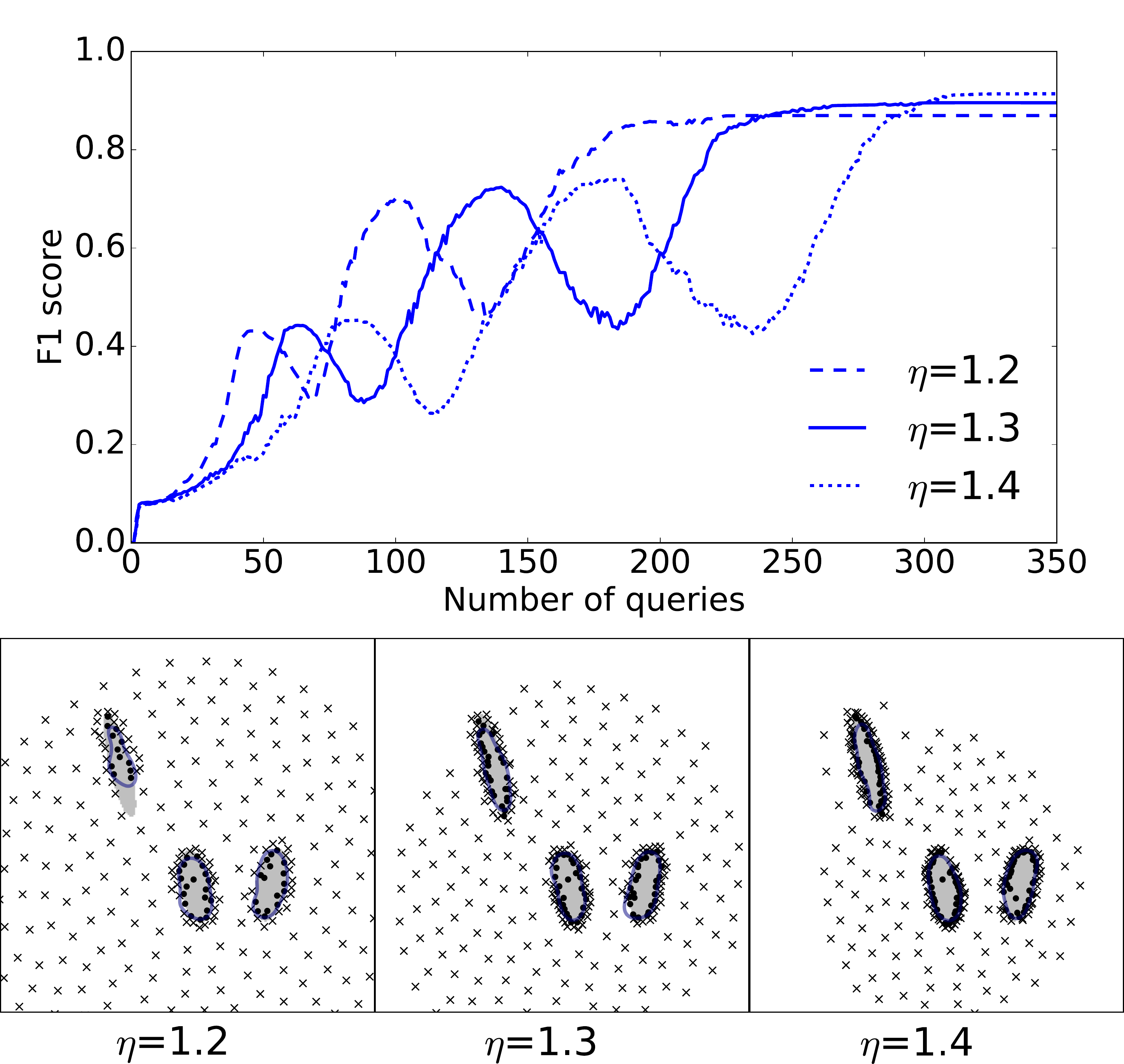}
\label{fig:f1s_eta}}
\caption{AES with different $\epsilon$ and $\eta$ on the Branin example. The upper plots show their F1 scores averaged over 100 runs. The lower plots show queried points during one of the 100 runs.}
\label{fig:hyppara}
\end{figure*}



\begin{figure}
\centering
\includegraphics[width=0.5\textwidth]{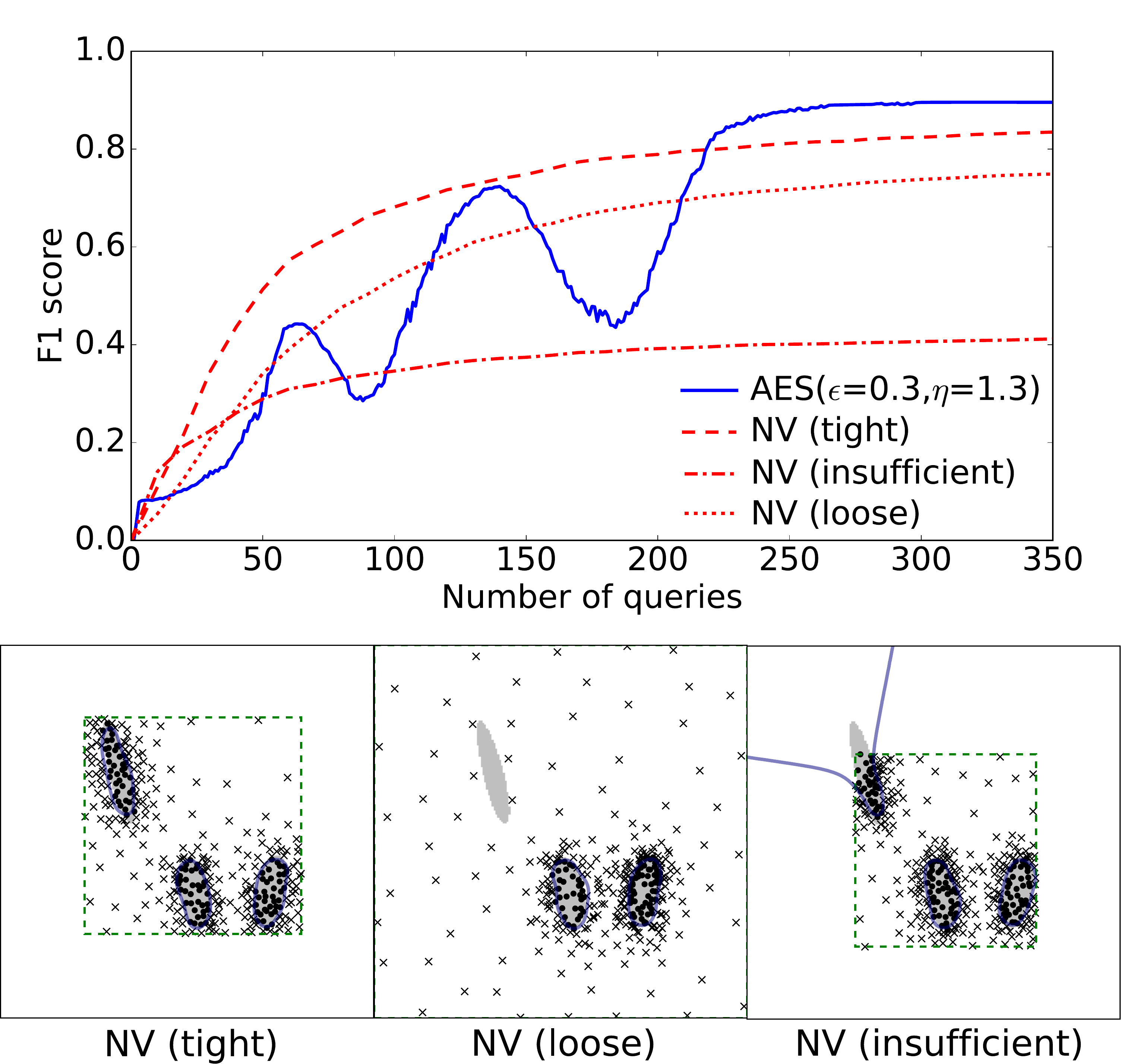}
\caption{AES and NV (with different input variable bounds) on the Branin example.  The pool boundaries set in the Neighborhood-Voronoi algorithm are shown as dashed lines.}
\label{fig:f1s_nv}
\end{figure}

\begin{figure}
\centering
\includegraphics[width=0.5\textwidth]{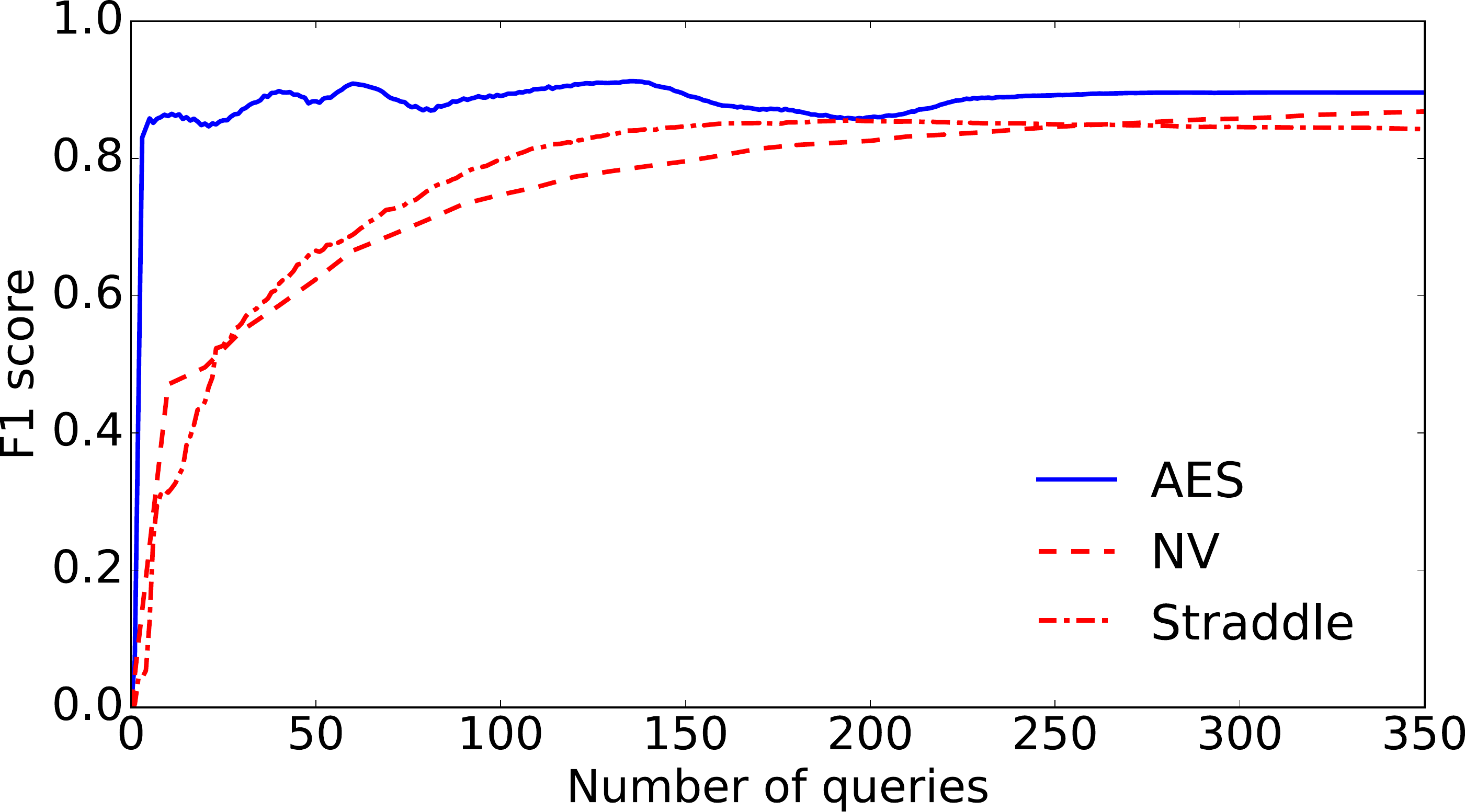}
\caption{F1 scores of AES ($\epsilon=0.3$ and $\eta=1.3$), NV, and Straddle (with tight bounds) on the Branin example within the explored region (\ie, the $p_\epsilon(\bm{x})<\tau$ region, where $\tau=\Phi(-\eta\epsilon)$, $\epsilon=0.3$, and $\eta=1.3$).}
\label{fig:f1s_explored}
\end{figure}



\begin{figure*}
\centering
\subfloat[Bernoulli noise.]{
\includegraphics[width=0.5\textwidth]{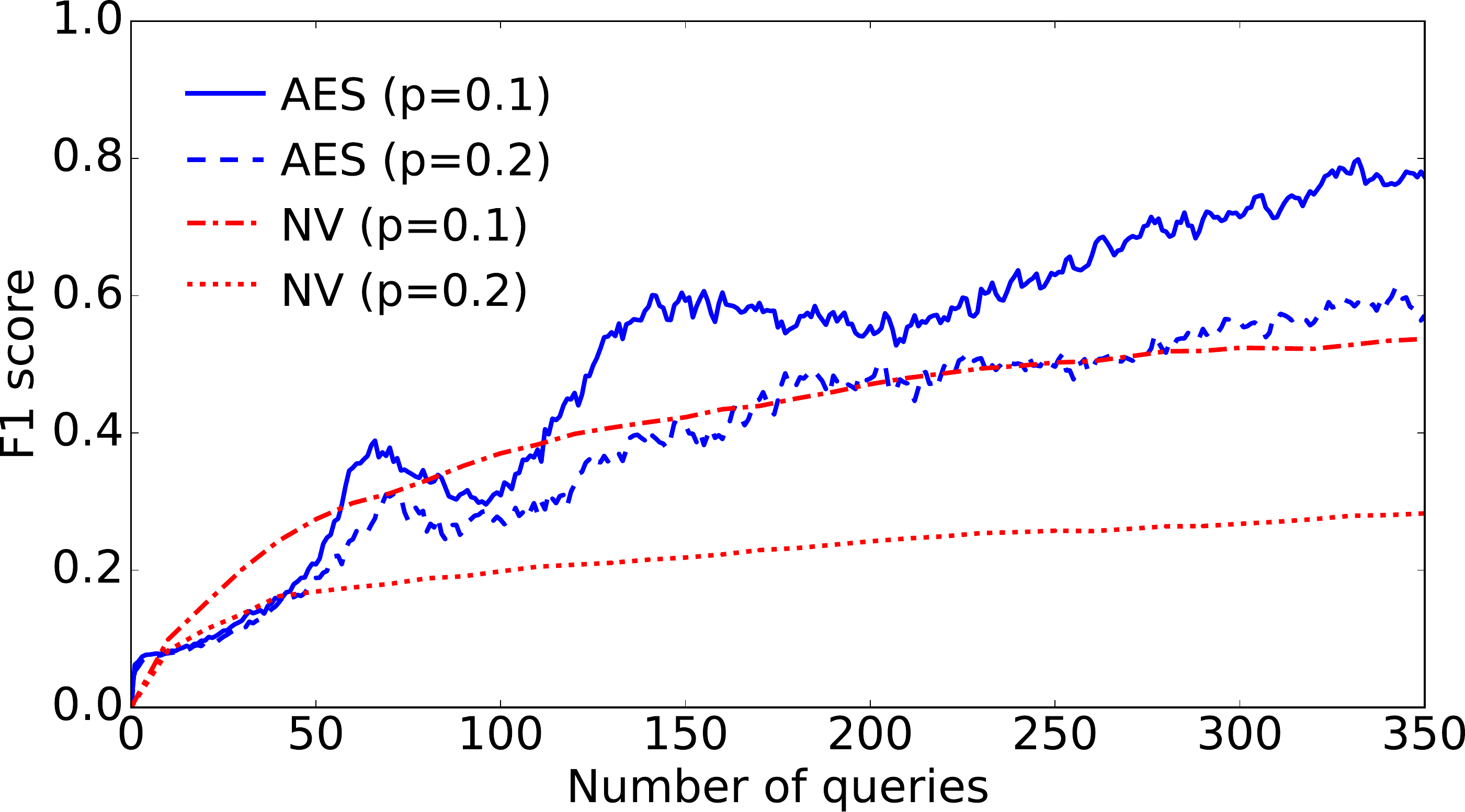}
\label{fig:f1s_bern}}
\subfloat[Gaussian noise.]{
\includegraphics[width=0.5\textwidth]{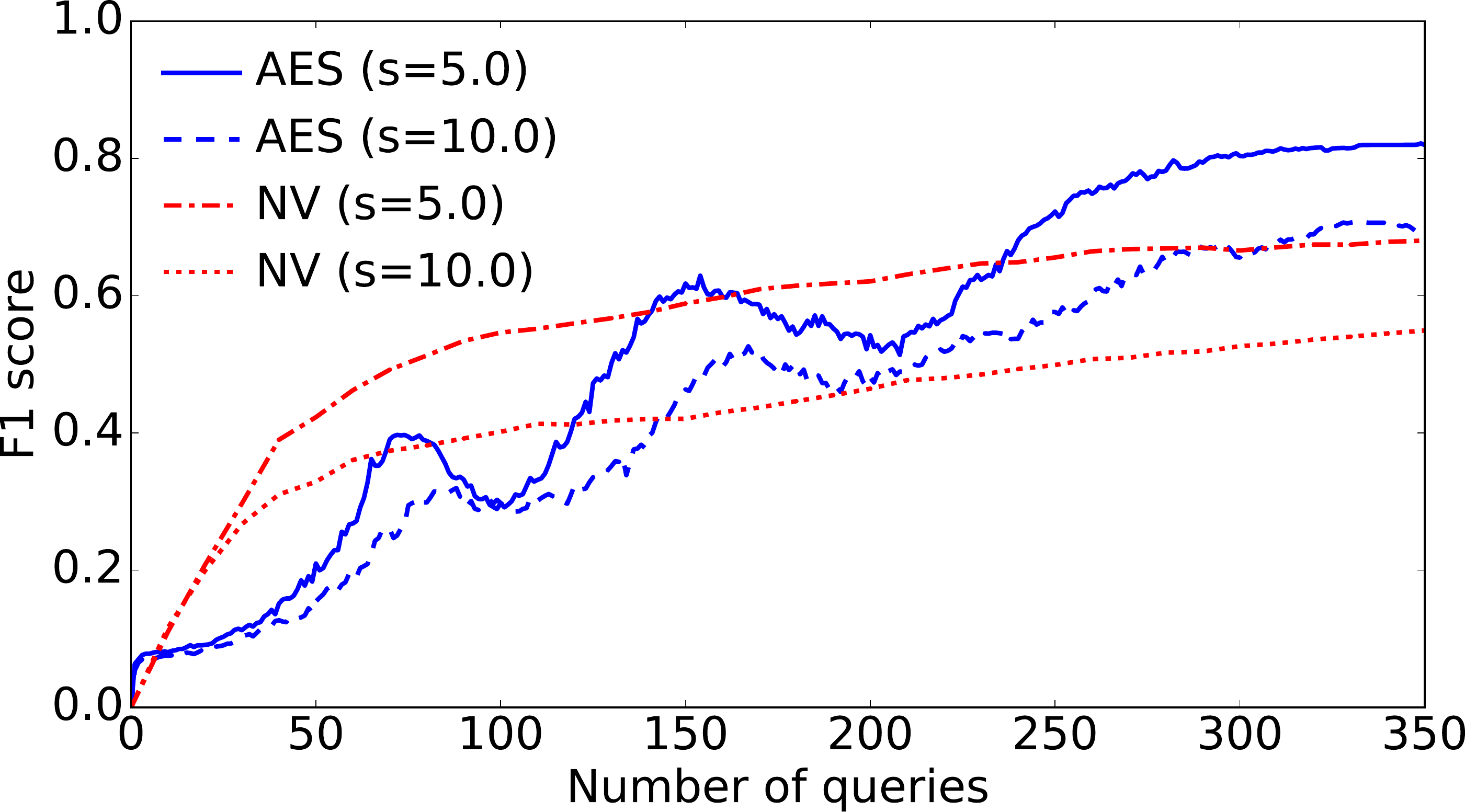}
\label{fig:f1s_gauss}}
\caption{AES and NV on the Branin example using noisy labels.}
\label{fig:f1s_noisy}
\end{figure*}


\begin{figure}
\centering
\includegraphics[width=0.5\textwidth]{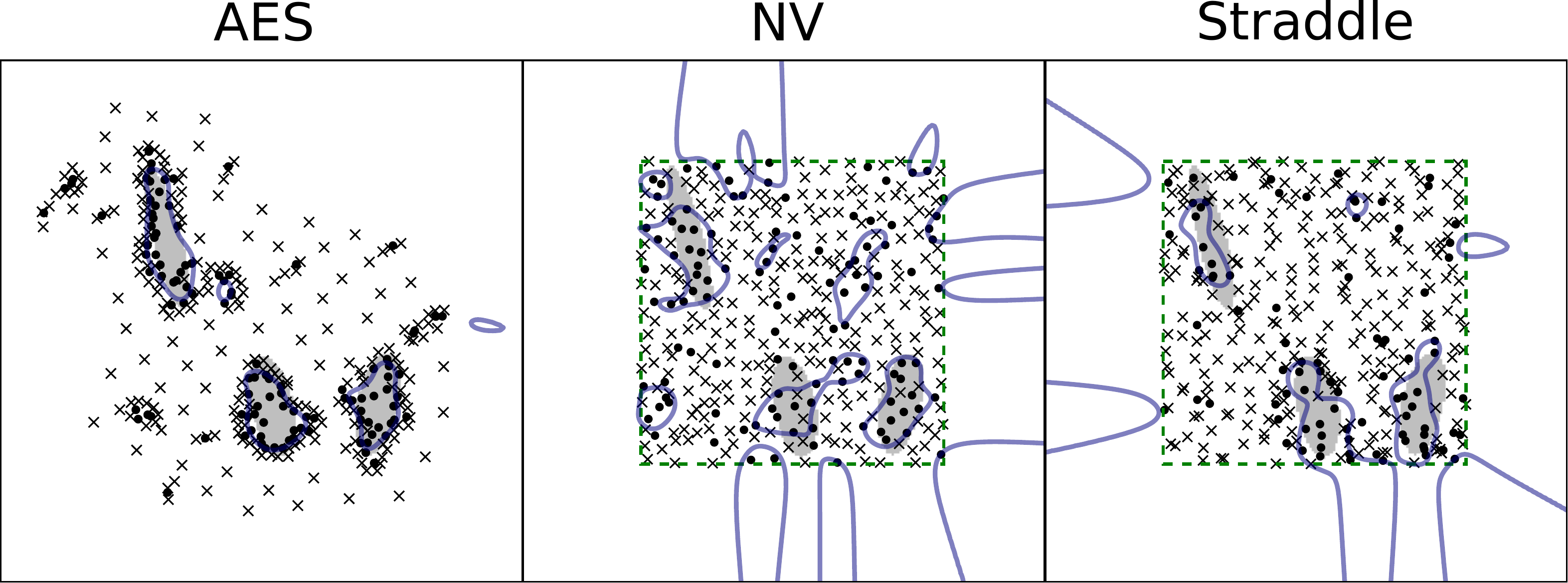}
\caption{Queried points under uniform label noise ($p=0.2$).}
\label{fig:noise}
\end{figure}

\begin{figure}
\centering
\includegraphics[width=0.5\textwidth]{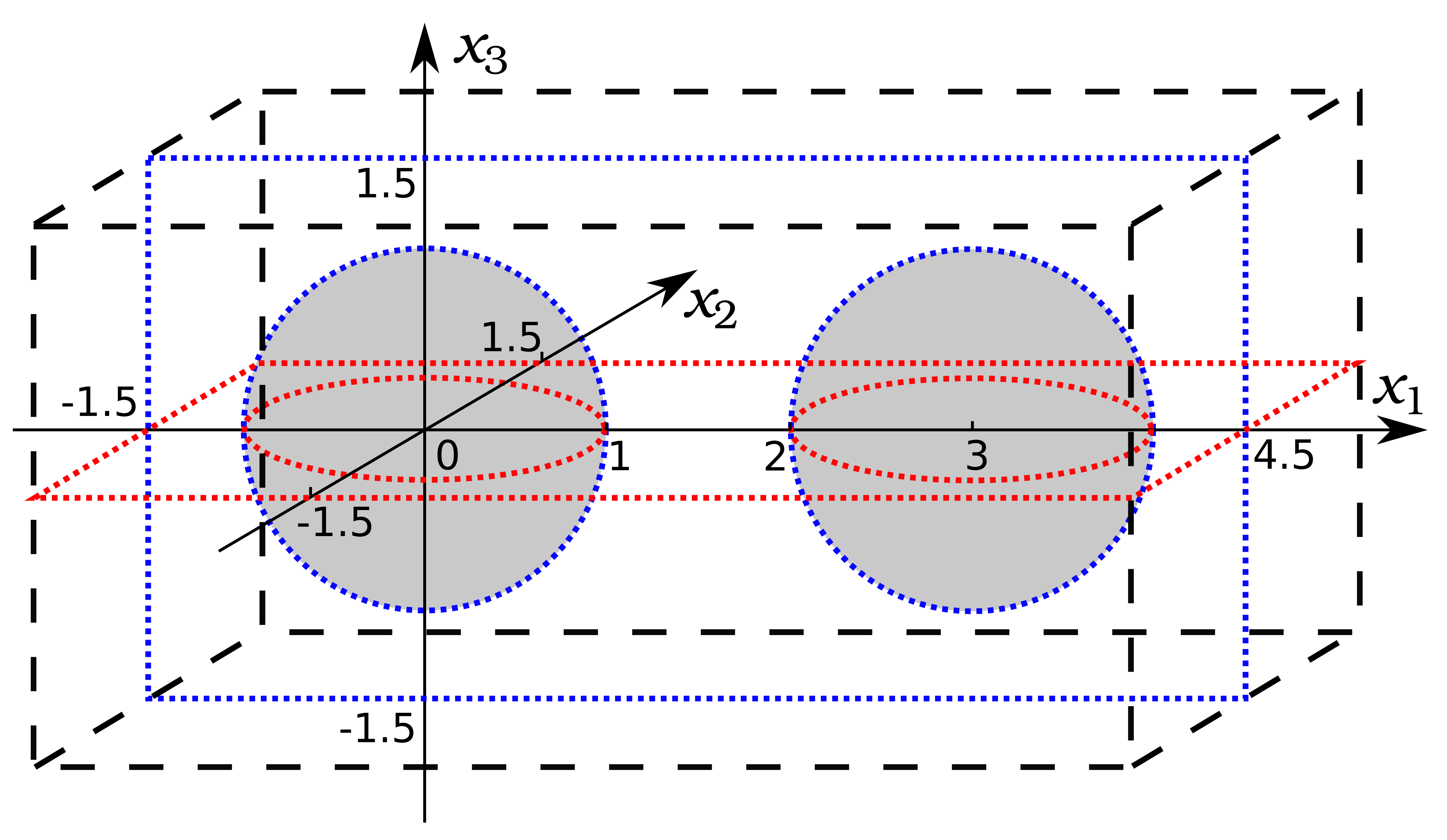}
\caption{The 3-dimensional double-sphere example. The gray regions are the feasible domains. The dashed boxes are the input space bounds for the NV algorithm.}
\label{fig:3d}
\end{figure}

\begin{figure*}
\centering
\subfloat[F1 scores.]{
\includegraphics[width=0.5\textwidth]{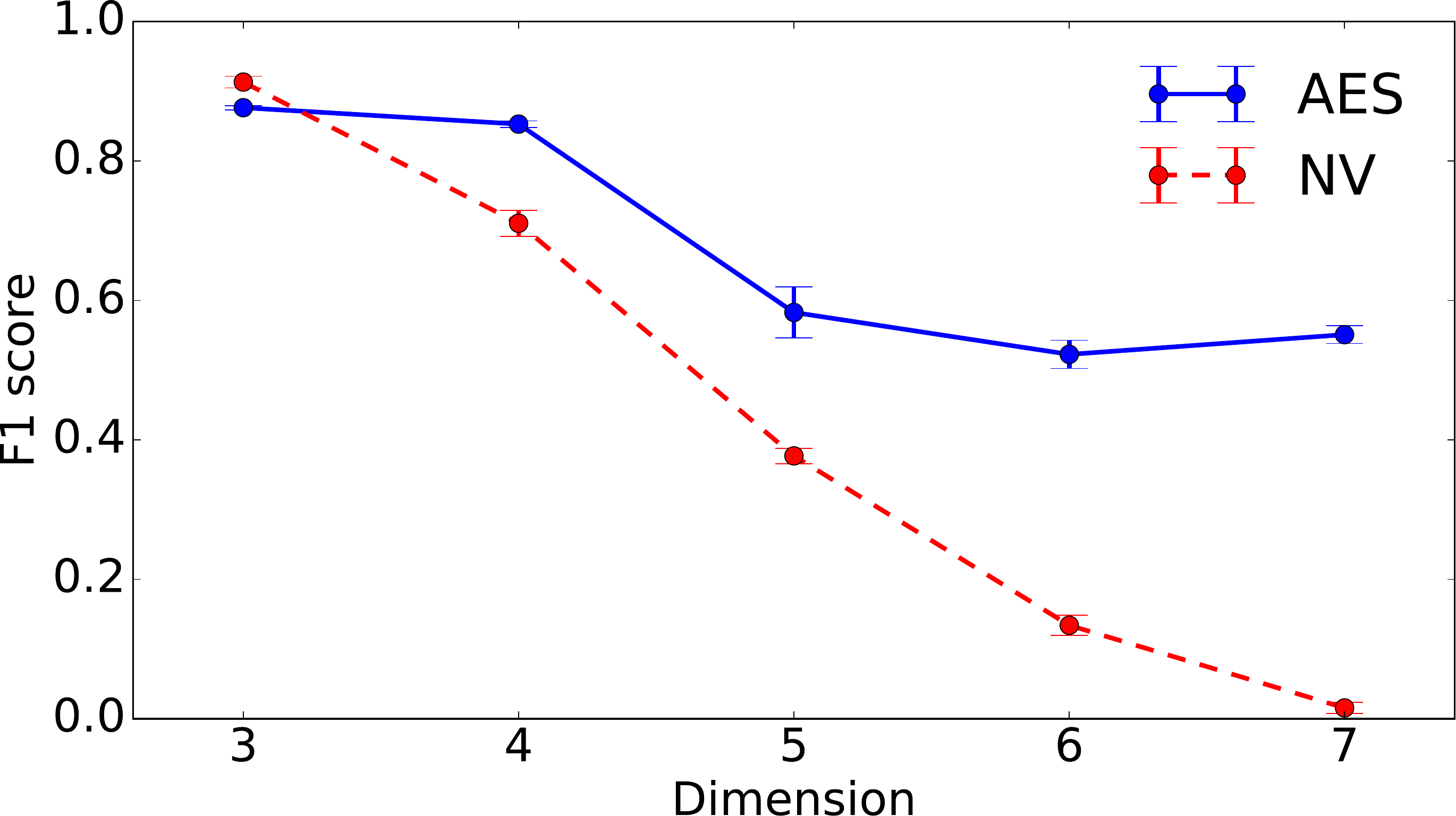}
\label{fig:highdim_f1s}}
\subfloat[Total running time.]{
\includegraphics[width=0.5\textwidth]{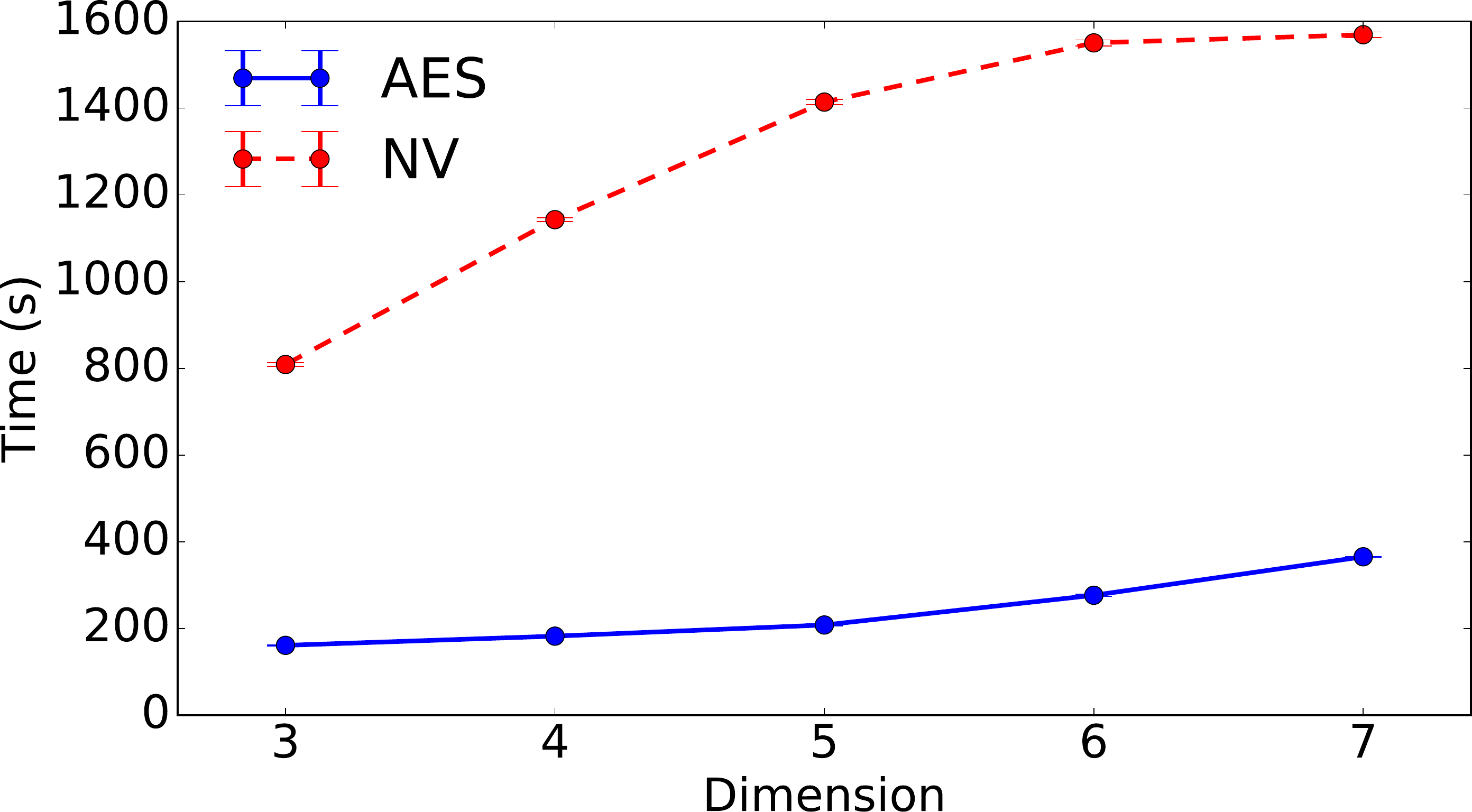}
\label{fig:highdim_time}}
\caption{AES and NV on high-dimensional double-sphere examples.}
\label{fig:highdim}
\end{figure*}

\begin{figure}
\centering
\includegraphics[width=0.4\textwidth]{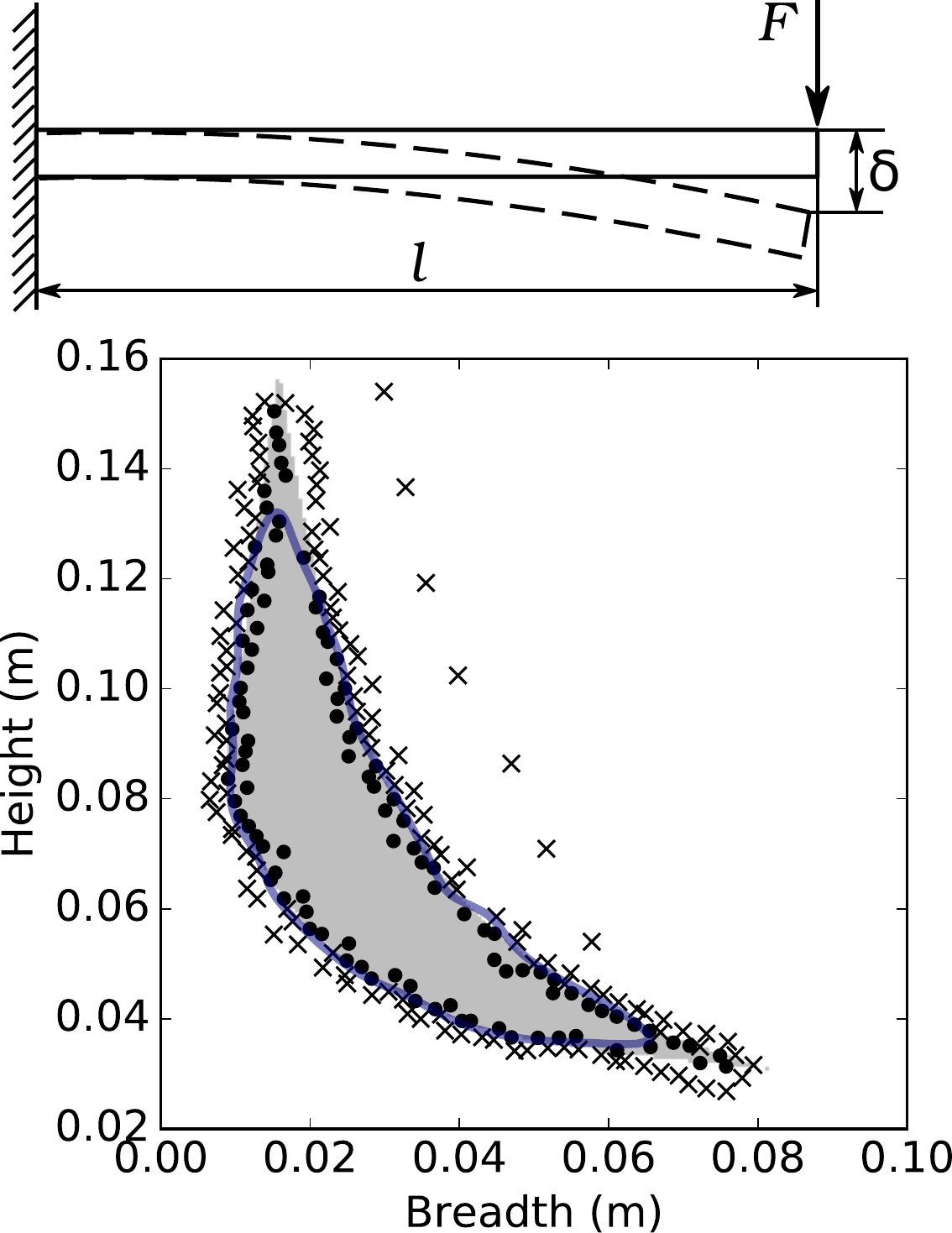}
\caption{AES on the Nowacki beam example.}
\label{fig:beam}
\end{figure}

\subsection{Effect of Hyperparameters}

We first use two 2-dimensional test functions\smoemdash the Branin function and Hosaki function, respectively\smoemdash as indicator functions to evaluate whether an input is inside the feasible domain. Both examples construct an input space with multiple disconnected feasible regions, which makes the feasible domain identification task challenging.

The Branin function is
\begin{equation*}
\begin{split}
g(\bm{x}) = &\left(x_2-\frac{5.1}{4\pi^2}x_1^2+\frac{5}{\pi}x_1-6\right)^2 \\
&+10\left(1-\frac{1}{8\pi}\right)\cos x_1+10
\end{split}
\end{equation*}
We define the label $y=1$ if $\bm{x} \in \{\bm{x}|g(\bm{x})\leq 8, -9<x_1<14, -7<x_2<17\}$; and $y=-1$ otherwise. The resulting feasible domains resemble three isolated feasible regions (Fig.~\ref{fig:opt}). The initial point $\bm{x}^{(0)} = (3,3)$. For the Gaussian process, we use a Gaussian kernel (Eq.~\ref{eq:gaussian_kernel}). We set the kernel length scale $l=0.9$. To compute the F1 scores, we generate samples along a $100 \times 100$ grid as the test set in the region where $x_1 \in [-13,18]$ and $x_2 \in [-8,23]$.

This section mainly describes the Branin example\smoemdash as both the Branin and Hosaki examples show similar results\smoemdash however we direct interested readers to the appendix (Sect.~\ref{exp:hosaki}) where we describe the Hosaki example in detail and show its experimental results.

For both examples, we use three levels of $\epsilon$ (0.1, 0.3, 0.5) and $\eta$ (1.2, 1.3, 1.4) to demonstrate their effects on AES's performance.

Figures~\ref{fig:opt} and \ref{fig:opt_bounded} show the sequence of queries selected by AES and the two bounded adaptive sampling methods, respectively, applied on the Branin example. For AES, there are three exploitation stages, as there are three disconnected feasible domains. AES starts by querying samples along the initial estimated decision boundary, and then expands queries outward to discover other feasible regions. In contrast, the straddle heuristic simultaneously explores the whole bounded input space, and refines all three decision boundaries. Fig.~\ref{fig:opt_f1s} shows the corresponding F1 scores of the experiment in Fig.~\ref{fig:opt}. During exploitation stages, AES's F1 score non-monotonically increases as part of the estimated decision boundary is outside the explored region (where AES has confidence on the accuracy); while in the exploration stage, the current decision boundaries are inside the explored region and remain unchanged, thus the F1 score stabilizes.

Figures~\ref{fig:f1s_epsilon} and \ref{fig:f1s_eta} demonstrate the effects of hyperparameters $\epsilon$ and $\eta$, respectively, on AES's performance. Increasing $\epsilon$ or $\eta$ leads to a slower expansion of the explored region and a higher F1 score. This means that using a higher $\epsilon$ or $\eta$ enables accuracy improvement but requires larger query budget. In both examples, the F1 score is more sensitive to $\eta$ than $\epsilon$.

\subsection{Unbounded versus Bounded}

We use the NV algorithm and the straddle heuristic as examples of bounded adaptive sampling methods. 
Because these two methods do not progressively expand the region (as in AES), but rather assumes a fixed region, we create a ``bounding box'' in the input space, and generate queries inside this box.


\begin{table}
\caption{Input space bounds for the NV algorithm and the straddle heuristic (Branin example).}
\label{tab:bounds}       
\begin{tabular}{llll}
\hline\noalign{\smallskip}
 & Tight & Loose & Insufficient  \\
\noalign{\smallskip}\hline\noalign{\smallskip}
Branin & \begin{tabular}{@{}l@{}}$x_1\in[-9,14]$,\\$x_2\in[-7,17]$\end{tabular} 
		& \begin{tabular}{@{}l@{}}$x_1\in[-14,19]$,\\$x_2\in[-12,22]$\end{tabular} 
        & \begin{tabular}{@{}l@{}}$x_1\in[-4,9]$,\\$x_2\in[-2,12]$\end{tabular} \\
\noalign{\smallskip}\hline
\end{tabular}
\end{table}

When comparing AES with the bounded methods, we use $\epsilon=0.3$ and $\eta=1.3$ for AES. In each experiment, we change the size of the input space bounds to evaluate the effect of bound size on these methods. 
Specifically, we simulate the cases where we set \textit{tight}, \textit{loose}, and \textit{insufficient} bounds, as shown in Tab.~\ref{tab:bounds} and Fig.~\ref{fig:f1s_nv}. ``Tight'' means the bounds cover the entire feasible domain while being no larger than needed (in practice we use bounds slightly larger than this to ensure the feasible domain boundary is inside the tight bounds); ``loose'' means the bounds cover the entire feasible domain but are larger than the tight bounds; ``insufficient'' means the variable bounds do not cover the entire feasible domain.

As shown in Fig.~\ref{fig:f1s_nv}, the NV algorithm makes fast accuracy improvement at early stages, and slows down after some iterations. The F1 score of NV is almost monotonically increasing; while AES's score fluctuates because it focuses first on refining the domains it knows about during exploitation (at the expense of accuracy on domains it has not seen yet). This causes AES to have a lower F1 score early on. For the NV algorithm, when the input variable bounds are set properly, both AES and NV achieve similar final F1 scores. However, NV requires more iterations to achieve a similar final accuracy to AES, especially when the bounds are set too large, where NV exhausts its query budget exploring unknown regions. When the bounds are set too small to cover certain feasible regions, NV stops improving the F1 score when it begins to over-sample the space and is unable to reach similar accuracy as AES. Note that in this case, we purposefully set the bounds such that they cover the vast majority of the feasible region, leaving only a small feasible area outside of those bounds. Our explicit purpose here is to demonstrate how sensitive such bounded heuristics can be when their bounds are misspecified (even by small amounts). The performance of bounded methods degrades rapidly as their bound sizes decrease further.

Although AES shows slow accuracy improvement over the entire test region, it keeps a constant accuracy bound within the explored region (as discussed in Sect.~\ref{sec:acc_bound}). Fig.~\ref{fig:f1s_explored} shows the F1 scores within the $p_\epsilon(\bm{x})<\tau$ region, which is AES's explored region. Specifically, we set $\epsilon=0.3$, $\eta=1.3$, and $\tau=\Phi(-\eta\epsilon)$. For the NV algorithm, we use the tight input space bounds from the previous experiments. By just considering the explored region, AES's F1 scores are quite stable throughout the sampling sequence; while NV's F1 scores are low at the beginning, and then increase until stable.\footnote{This difference is because NV's explored region covers more area than AES at the beginning.} Since AES's accuracy inside the explored region is invariant of the number of iterations, it can be used for real-time prediction of samples' feasibility in the explored region.

\begin{table}
\caption{Final F1 scores and running time (Branin example).}
\label{tab:f1s_time}
\begin{tabular}{llrr}
\hline\noalign{\smallskip}
 & & F1 score & Time (s)  \\
\noalign{\smallskip}\hline\noalign{\smallskip}
\multirow{11}{*}{\begin{sideways} Branin (350 queries) \end{sideways}} & AES ($\epsilon=0.3,\eta=1.3$)  & $0.90\pm0.004$ & $92.34\pm0.62$ \\
& AES ($\epsilon=0.1,\eta=1.3$)  & $0.87\pm0.008$ & $95.71\pm0.37$ \\
& AES ($\epsilon=0.5,\eta=1.3$)  & $0.90\pm0.002$ & $89.71\pm0.38$ \\
& AES ($\epsilon=0.3,\eta=1.2$)  & $0.87\pm0.006$ & $96.73\pm0.26$ \\
& AES ($\epsilon=0.3,\eta=1.4$)  & $0.91\pm0.002$ & $80.70\pm0.33$ \\
& NV (tight) & $0.83\pm0.021$ & $64.40\pm0.09$ \\
& NV (loose) & $0.75\pm0.030$ & $63.68\pm0.06$ \\
& NV (insufficient) & $0.41\pm0.028$ & $63.83\pm0.06$ \\
& Straddle (tight) & $0.82\pm0.012$ & $43.72\pm0.22$ \\
& Straddle (loose) & $0.71\pm0.014$ & $41.72\pm0.22$ \\
& Straddle (insufficient) & $0.34\pm0.009$ & $54.44\pm0.21$ \\
\noalign{\smallskip}\hline
\end{tabular}
\end{table}

Table~\ref{tab:f1s_time} shows the final F1 scores and wall-clock running time of AES, NV, and the straddle heuristic. Note that the confidence interval for NV's averaged F1 scores are much larger than AES. This is because during some runs NV fails to discover all the three feasible regions (Fig.~\ref{fig:f1s_nv} for example).

\subsection{Effect of Noise}

Label noise is usually inevitable in active learning tasks. The noise comes from, for example, simulation/experimental error or human annotators' mistakes. We test the cases where the labels are under (1)~uniform noise and (2)~Gaussian noise centered at the decision boundary. 

We simulate the first case by randomly flipping the labels. The noisy label is set as $y' = (-1)^{\lambda} y$, where $\lambda \sim \text{Bernoulli}(p)$, $p$ is the parameter of the Bernoulli distribution that indicates the noise level, and $y$ is the true label. 

The second case is probably more common in practice, since it is usually harder to decide the labels near the decision boundary. To simulate this case, we add Gaussian noise to the test functions: $g'(\bm{x})=g(\bm{x})+e$, where $g(\bm{x})$ is the Branin or Hosaki function, and $e \sim s\cdot\mathcal{N}(0, 1)$.

In each case we compare the performance of AES ($\epsilon=0.3, \eta=1.3$) and NV (with tight bounds) under two noise levels. As expected, adding noise to the labels decreases the accuracy of both methods (Fig.~\ref{fig:f1s_bern} and \ref{fig:f1s_gauss}). However, in both cases (Bernoulli noise and Gaussian noise), the noise appears to influence NV more than AES. As shown in Fig.~\ref{fig:noise}, when adding noise to the labels, NV has high error mostly along the input space boundaries, where it cannot query samples outside to further investigate those apparent feasible regions. In contrast, AES tries to exploit those rogue points to try to find new feasible regions, realizing after a few new samples that they are noise.

\subsection{Effect of Dimensionality}

To test the effects of dimensionality on AES's performance, we apply both AES and NV on higher-dimensional examples where the feasible domains are inside two $(d-1)$-spheres of radius 1 centered at $\bm{a}$ and $\bm{b}$ respectively. Here $\bm{a}=\bm{0}$ and $\bm{b}=(3,0,...,0)$. Fig.~\ref{fig:3d} shows the input space of the 3-dimensional double-sphere example. The initial point $\bm{x}^{(0)} = \bm{0}$. For the Gaussian process, we use a Gaussian kernel with a length scale of 0.5. We set $\epsilon=0.3$ and $\eta=1.3$. To compute the F1 scores, we randomly generate 10,000 samples uniformly within the region where $x_1 \in [-2,5]$ and $x_k \in [-2,2], k=2,...,d$. The input space bounds for the NV algorithm are $x_1 \in [-1.5,4.5]$ and $x_k \in [-1.5,1.5], k=2,...,d$. We get the F1 scores and running time after querying 1,000 points.

As shown in Fig.~\ref{fig:highdim}, both AES and NV shows an accuracy drop and running time increase as the problem's dimensionality increases. This is expected, since based on the curse of dimensionality~\citep{bellman1957dynamic}, the number of queries needed to achieve the same accuracy increases with the input space dimensionality. The curse of dimensionality is inevitable in machine learning problems. However, since AES explores the input space only when necessary (\ie, only after it has seen the entire decision boundary of the discovered feasible domain), its queries do not need to fill up the large volume of the high-dimensional space. Therefore, AES's accuracy drop with problem dimensionality is not as severe as bounded methods like NV. For particularly high-dimensional design problems, another complementary approach is to construct explicit lower-dimensional design manifolds upon which to run AES~\citep{chen2017design,chen2017beyond}.

\subsection{Nowacki Beam Example}

To test AES's performance in a real-world scenario, we consider the Nowacki beam problem~\citep{nowacki1980modelling}. The original Nowacki beam problem is a design optimization problem where we minimize the cross-section area $A$ of a cantilever beam of length $l$ with concentrated load $F$ at its end. The design variables are the beam's breadth $b$ and height $h$. We turn this problem into a feasible domain identification problem by replacing the objective with a constraint $A = bh \leq 0.0025 \si{m}^2$. Other constraints are (1)~the maximum tip deflection $\delta = Fl^3/(3EI_Y) \leq 5\si{mm}$, (2)~the maximum blending stress $\sigma_B = 6Fl/(bh^2) \leq \sigma_Y$, (3)~the maximum shear stress $\tau = 1.5F/(bh) \leq \sigma_Y/2$, (4)~the ratio $h/b \leq 10$, and (5)~the failure force of buckling $F_{crit} = (4/l^2)\sqrt{(G I_T)(E I_Z)/(1-\nu^2)} \geq fF$, where $I_Y=bh^3/12$, $I_Z=b^3h/12$, $I_T=I_Y+I_Z$, and $f$ is the safety factor. And $\sigma_Y$, $E$, $\nu$, and $G$ are the yield stress, Young's modulus, Poisson's ratio, and shear modulus of the beam's material, respectively. We use the settings from \cite{singh2017sequential}, where $l=0.5\si{m}$, $F=5\si{kN}$, $f=2$, $\sigma_Y=240\si{MPa}$, $E=216.62\si{GPa}$, $\nu=0.27$, and $G=86.65\si{GPa}$. As shown in Fig.~\ref{fig:beam}, the feasible domain is a crescent-shaped region. Given only these constraints, it is unclear what appropriately tight bounds on the design variables should be.

In this experiment, we set the Gaussian kernel's length scale as 0.005, $\epsilon=0.3$ and $\eta=1.3$. The initial point $\bm{x}^{(0)}=(b_0,h_0)=(0.05,0.05)$. The test samples are generated along a $100 \times 100$ grid in the region where $b \in [0,0.02]$ and $h \in [0.1,0.16]$.

After 242 iterations, the F1 score of AES reaches 0.933 and remains constant. Note that mostly the estimation error comes from the two sharp ends of the crescent-shaped feasible region (Fig.~\ref{fig:beam}). This is because the kernel's assumption on function smoothness (\ie, similar inputs should have similar outputs) causes the GP to have bad performance where the labels shift frequently. The similar problem also exists when using other classifiers like SVM, where a kernel is also used to enforce similar outputs between similar inputs. This problem can be alleviated by using a smaller kernel length scale.

\section{Conclusion}

We presented a pool-based sampling method, AES, for identifying (possibly disconnected) feasible domains over an unbounded input space. Unlike conventional methods that sample inside a fixed boundary, AES progressively expands our knowledge of the input space under an accuracy guarantee. We showed that AES uses successive exploitation and exploration stages to switch between learning the decision boundary and searching for new feasible domains. To avoid increasing the pool size and hence the computation cost as the explored area grows, we proposed a dynamic local pool generation method that samples the pool locally at a certain location in each iteration.
We showed that at any point within the explored region, AES guarantees an upper bound $\epsilon$ of misclassification loss with a probability of at least $1-\tau$, regardless of the number of iterations or labeled samples. This means that AES can be used for real-time prediction of samples' feasibility inside the explored region. We also demonstrated that, compared to existing methods, AES can achieve comparable or higher accuracy without needing to set exact bounds on the input space.

Note that AES cannot be applied on input spaces where synthesizing a useful sample is difficult. For example, in an image classification task, we cannot directly synthesize an image by arbitrarily setting its pixels, since most of the synthesized images may be unrealistic and hence useless. Usually in such cases, we use real-world samples as the pool and apply bounded active learning methods (since we know the bounds of real-world samples). Or instead, we first embed the original inputs onto a lower-dimensional space, such that given the low-dimensional representation, we can synthesize realistic samples. We can then apply AES on that embedded space. This approach can be used for discovering novel feasible domains (\ie, finding feasible inputs that are nonexistent in the real-world). We refer interested readers to a detailed introduction of this approach by \cite{chen2017beyond}.

One limitation of AES is that the accuracy improves slowly at the early stage compared to bounded active learning methods. This is because AES focuses on only the explored region (which is small at the beginning), while bounded active learning methods usually do space-filling at first. In the situation where we want fast accuracy improvement at the beginning, one possible way of tackling this problem is by dynamically setting AES's hyperparameters. Specifically, since the expansion speed increases as $\epsilon$ or $\eta$ decreases, we can accelerate AES's accuracy improvement at earlier stages by setting small values of $\epsilon$ and $\eta$, so that queries quickly fill up a larger region. Then to achieve high final accuracy, we can increase $\epsilon$ and $\eta$ to meet the accuracy requirement.



\bibliographystyle{spbasic}      
\bibliography{references}   

%
%

\pagebreak
\begin{center}
\textbf{\Large Appendix A: Theorem Proofs}
\end{center}
\setcounter{equation}{0}
\setcounter{figure}{0}
\setcounter{table}{0}
\setcounter{section}{0}
\makeatletter
\renewcommand{\theequation}{A\arabic{equation}}
\renewcommand{\thefigure}{A\arabic{figure}}
\renewcommand{\thetable}{A\arabic{table}}
\renewcommand{\thesection}{A\arabic{section}}
\renewcommand{\bibnumfmt}[1]{[A#1]}
\renewcommand{\citenumfont}[1]{A#1}

\section{Proof of Theorem~\ref{thm:delta}}
\label{pf:delta}

According to Eq.~\ref{eq:mean}, given an optimal query $\bm{x}^*$, we have
\begin{equation*}
\begin{aligned}
|\bar{f}(\bm{x}^*)| &= |\bm{k}(\bm{x}^*)^T \nabla \log p(\bm{y}|\hat{\bm{f}})| \\
&= \left|\bm{k}(\bm{x}^*)^T \nabla \log \begin{bmatrix}
    \Phi(y_1f_1)\\
    \vdots\\
    \Phi(y_{t-1}f_{t-1})
\end{bmatrix}\right| \\
&= \left|\bm{k}(\bm{x}^*)^T \begin{bmatrix}
    y_1\mathcal{N}(f_1)/\Phi(y_1f_1)\\
    \vdots\\
    y_{t-1}\mathcal{N}(f_{t-1})/\Phi(y_{t-1}f_{t-1})
\end{bmatrix}\right| \\
&= \left| \sum_{i=1}^{t-1} k(\bm{x}^*,\bm{x}^{(i)})y_i\frac{\mathcal{N}(f_i)}{\Phi(y_if_i)} \right| \\
&\leq \sum_{i=1}^{t-1} \left| k(\bm{x}^*,\bm{x}^{(i)})y_i\frac{\mathcal{N}(f_i)}{\Phi(y_if_i)} \right| \\
&< \sum_{i=1}^{t-1} k_m \text{sign}(y_i)y_i\frac{\mathcal{N}(f_i)}{\Phi(y_if_i)} \\
&= k_m \text{sign}(\bm{y})^T \begin{bmatrix}
    y_1\mathcal{N}(f_1)/\Phi(y_1f_1)\\
    \vdots\\
    y_{t-1}\mathcal{N}(f_{t-1})/\Phi(y_{t-1}f_{t-1})
\end{bmatrix} \\
&= k_m \mu
\end{aligned}
\label{eq:f_bound}
\end{equation*}
where 
\begin{equation}
\begin{aligned}
k_m &= \max_{\bm{x}^{(i)}\in X_L} k(\bm{x}^*,\bm{x}^{(i)}) \\
&= \exp\left(-\frac{\min_{\bm{x}^{(i)}\in X_L} \|\bm{x}^*-\bm{x}^{(i)}\|^2}{2l^2}\right) \\
&= e^{-\delta^2/(2l^2)}
\label{eq:k_m}
\end{aligned}
\end{equation}
and
\begin{equation}
\mu = \text{sign}(\bm{y})^T \nabla \log p(\bm{y}|\hat{\bm{f}})
\label{eq:mu}
\end{equation}

Similarly,
\begin{equation}
\begin{aligned}
V(\bm{x}^*) &= 1-\bm{k}(\bm{x}^*)^T(K+W^{-1})^{-1}\bm{k}(\bm{x}^*) \\
&> 1-(k_m\bm{1})^T (K+W^{-1})^{-1} (k_m\bm{1}) \\
&= 1-k_m^2\bm{1}^T(K+W^{-1})^{-1}\bm{1} \\
&= 1-k_m^2 \nu
\label{eq:v_bound}
\end{aligned}
\end{equation}
where
\begin{equation}
\nu = \bm{1}^T(K+W^{-1})^{-1}\bm{1}
\label{eq:nu}
\end{equation}

Therefore for the optimal query $\bm{x}^*$ we have
\begin{equation*}
p_\epsilon(\bm{x}^*) = \Phi\left(-\frac{|\bar{f}(\bm{x}^*)|+\epsilon}{\sqrt{V(\bm{x}^*)}}\right) 
> \Phi\left(-\frac{k_m \mu+\epsilon}{\sqrt{1-k_m^2 \nu}}\right)
\end{equation*}

Both Theorem~\ref{thm:intersection} and \ref{thm:isocontour} state that $p_\epsilon(\bm{x}^*)=\tau$, thus
\begin{equation*}
\Phi\left(-\frac{k_m \mu+\epsilon}{\sqrt{1-k_m^2 \nu}}\right) < \tau
\end{equation*}

When $\tau=\Phi(-\eta\epsilon)$, we have
\begin{equation}
\frac{k_m \mu+\epsilon}{\sqrt{1-k_m^2 \nu}} > \eta\epsilon
\label{eq:pool_bound}
\end{equation}

Plugging Eq.~\ref{eq:k_m} into Eq.~\ref{eq:pool_bound} and solving for the distance $\delta$, we get
\begin{equation*}
\delta < \beta l
\label{eq:delta}
\end{equation*}
where
\begin{equation}
\beta=\sqrt{2\log\frac{\mu^2+\eta^2\epsilon^2\nu}{\eta\epsilon\sqrt{\mu^2+(\eta^2-1)\epsilon^2\nu}-\epsilon\mu}}
\label{eq:beta}
\end{equation}

\section{Proof of Theorem~\ref{thm:delta_exploit}}
\label{pf:delta_exploit}

Theorem~\ref{thm:intersection} states that the optimal query in the exploitation stage lies at the intersection of $\bar{f}(\bm{x})=0$ and $p_\epsilon(\bm{x})=\tau$. By substituting $\Phi(-\eta\epsilon)$ for $\tau$, we have
\begin{equation}
V(\bm{x}^*) = \frac{1}{\eta^2}
\label{eq:v_exploit1}
\end{equation}
According to Eq.~\ref{eq:v_bound}, we have $V(\bm{x}^*)>1-k_m^2\nu$. Combining Eq.~\ref{eq:k_m}, \ref{eq:nu}, and \ref{eq:v_exploit1}, we get 
\begin{equation*}
\delta < \delta_{exploit} = \gamma l
\label{eq:delta_exploit}
\end{equation*}
where
\begin{equation}
\gamma = \sqrt{\log\frac{\eta^2\nu}{\eta^2-1}}
\label{eq:gamma}
\end{equation}

\section{Proof of Theorem~\ref{thm:density}}
\label{pf:density}

According to Eq.~\ref{eq:v_exploit1}, the predictive variance of an optimal query $\bm{x}_{exploit}$ in the exploitation stage is
\begin{equation*}
V(\bm{x}_{exploit}) = \frac{1}{\eta^2}
\end{equation*}
While in the exploration stage, we have $p_\epsilon(\bm{x}_{explore})=\tau$ at the optimal query $\bm{x}_{explore}$ (Theorem~\ref{thm:isocontour}). And by applying Eq.~\ref{eq:p_epsilon} and setting $\tau = \Phi(-\eta\epsilon)$, we have
\begin{equation*}
V(\bm{x}_{explore}) = \frac{1}{\eta^2}\left(1+\frac{|\bar{f}(\bm{x}_{explore})|}{\epsilon}\right)^2
\end{equation*}

\pagebreak
\begin{center}
\textbf{\Large Appendix B: Additional Experimental Results}
\end{center}
\setcounter{equation}{0}
\setcounter{figure}{0}
\setcounter{table}{0}
\setcounter{section}{0}
\makeatletter
\renewcommand{\theequation}{B\arabic{equation}}
\renewcommand{\thefigure}{B\arabic{figure}}
\renewcommand{\thetable}{B\arabic{table}}
\renewcommand{\thesection}{B\arabic{section}}
\renewcommand{\bibnumfmt}[1]{[B#1]}
\renewcommand{\citenumfont}[1]{B#1}

\section{Hosaki Example}
\label{exp:hosaki}

\begin{figure*}
\centering
\subfloat[Changing $\epsilon$ ($\eta=1.3$).]{
\includegraphics[width=0.5\textwidth]{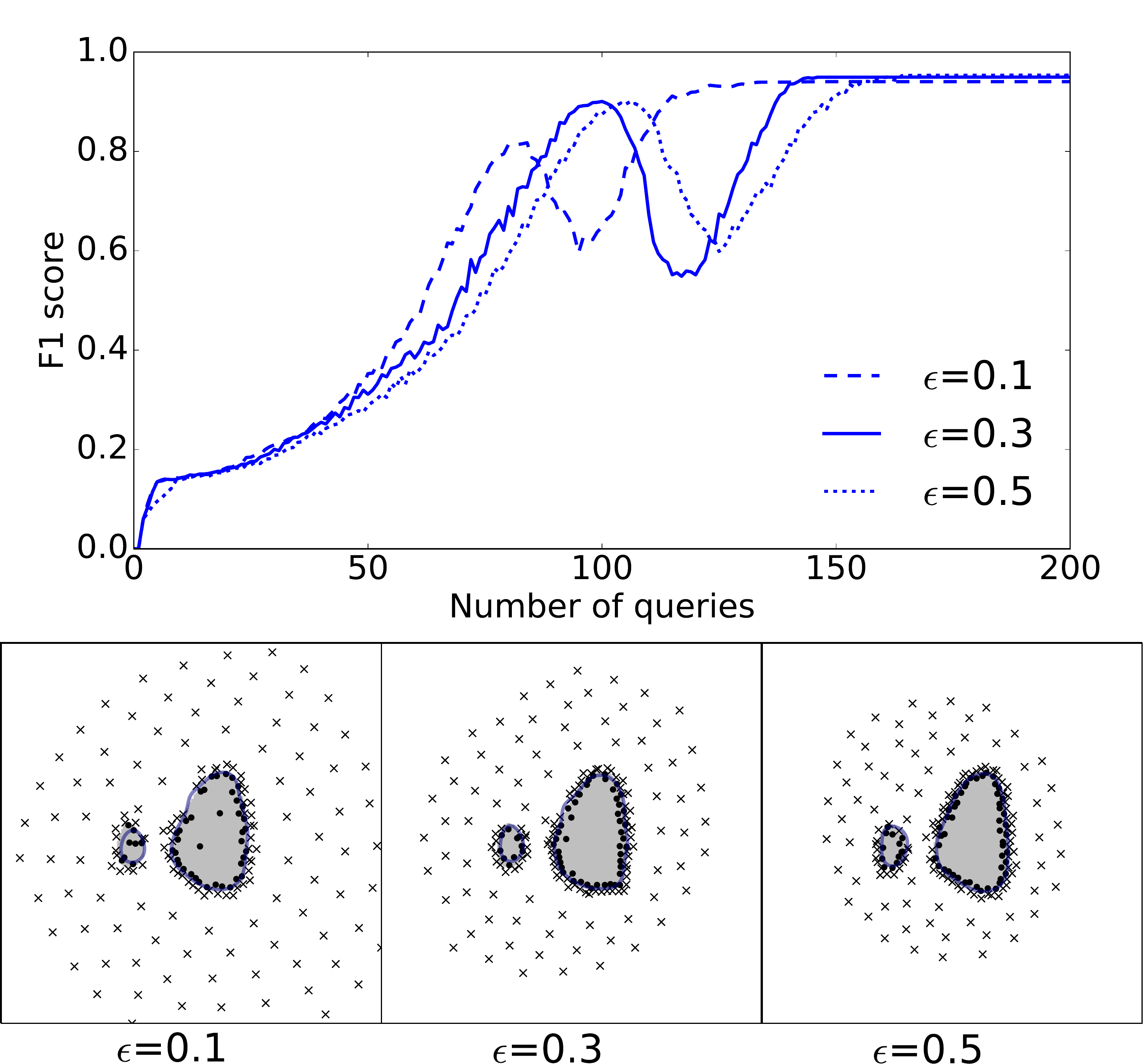}
\label{fig:f1s_epsilon_hosaki}}
\subfloat[Changing $\eta$ ($\epsilon=0.3$).]{
\includegraphics[width=0.5\textwidth]{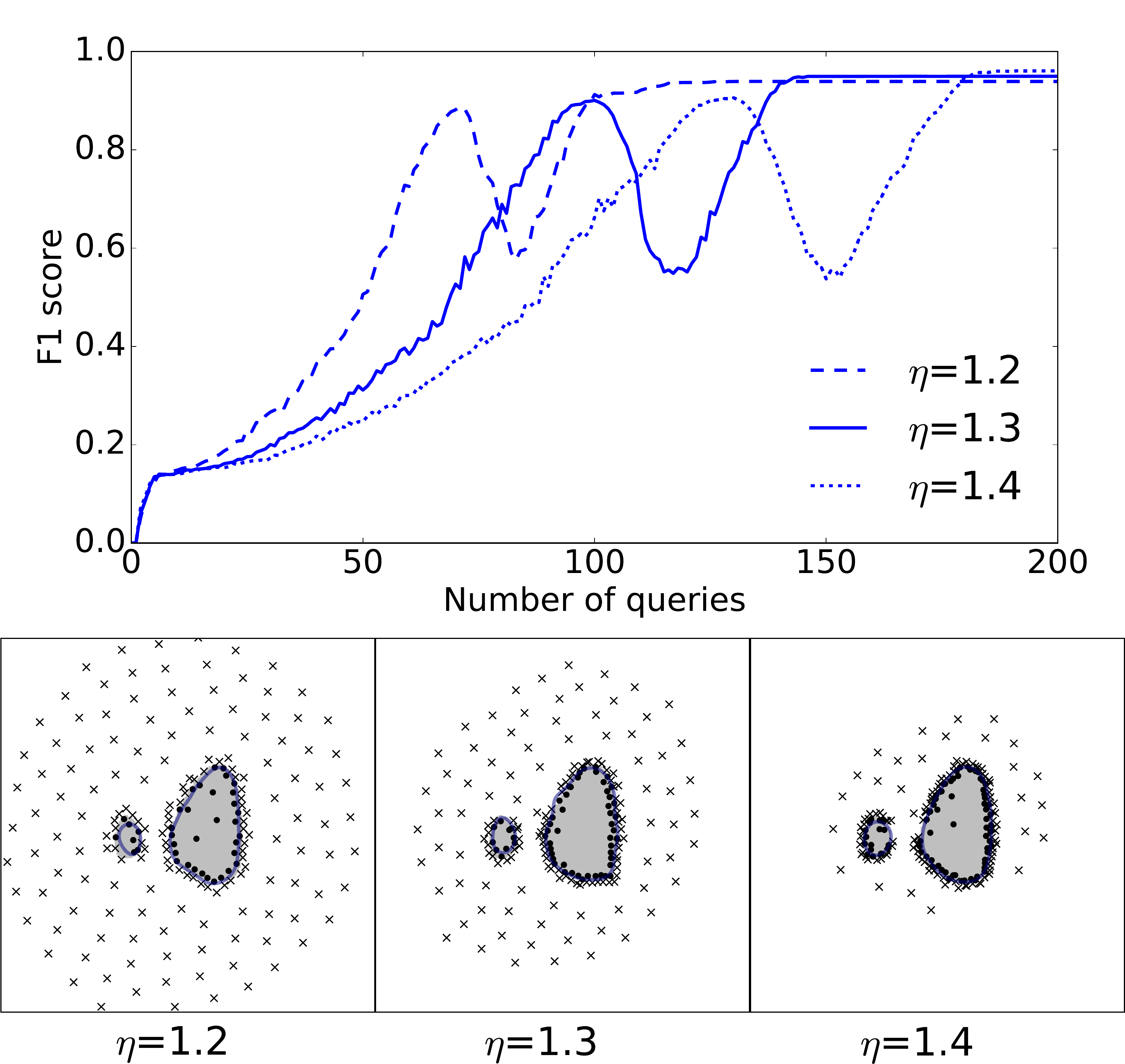}
\label{fig:f1s_eta_hosaki}}
\caption{AES with different $\epsilon$ and $\eta$ on the Hosaki example.}
\label{fig:hyppara_hosaki}
\end{figure*}

\begin{figure}
\centering
\includegraphics[width=0.5\textwidth]{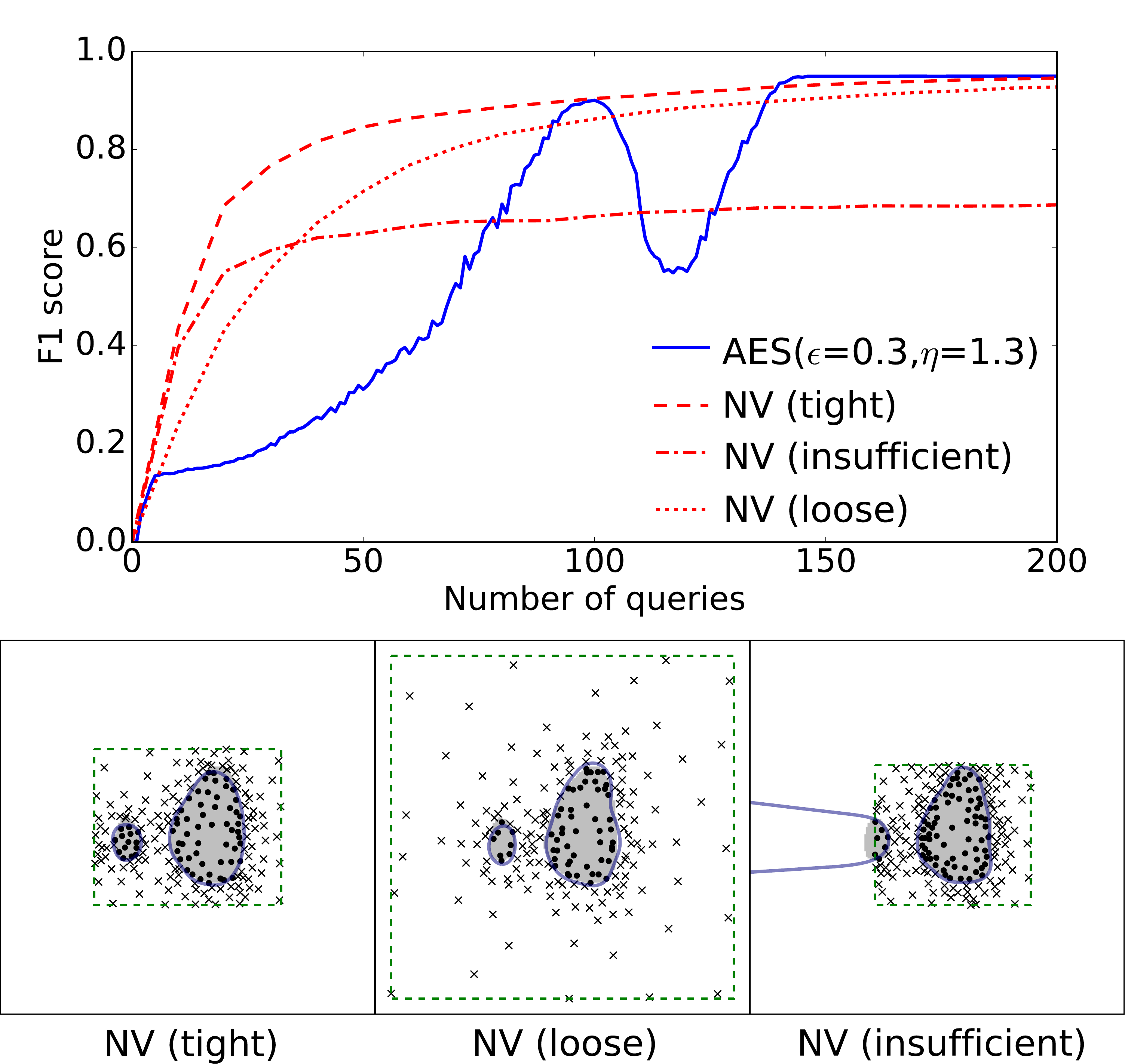}
\caption{AES and NV (with different input variable bounds) on the Hosaki example.}
\label{fig:f1s_nv_hosaki}
\end{figure}


\begin{figure*}
\centering
\subfloat[Bernoulli noise.]{
\includegraphics[width=0.5\textwidth]{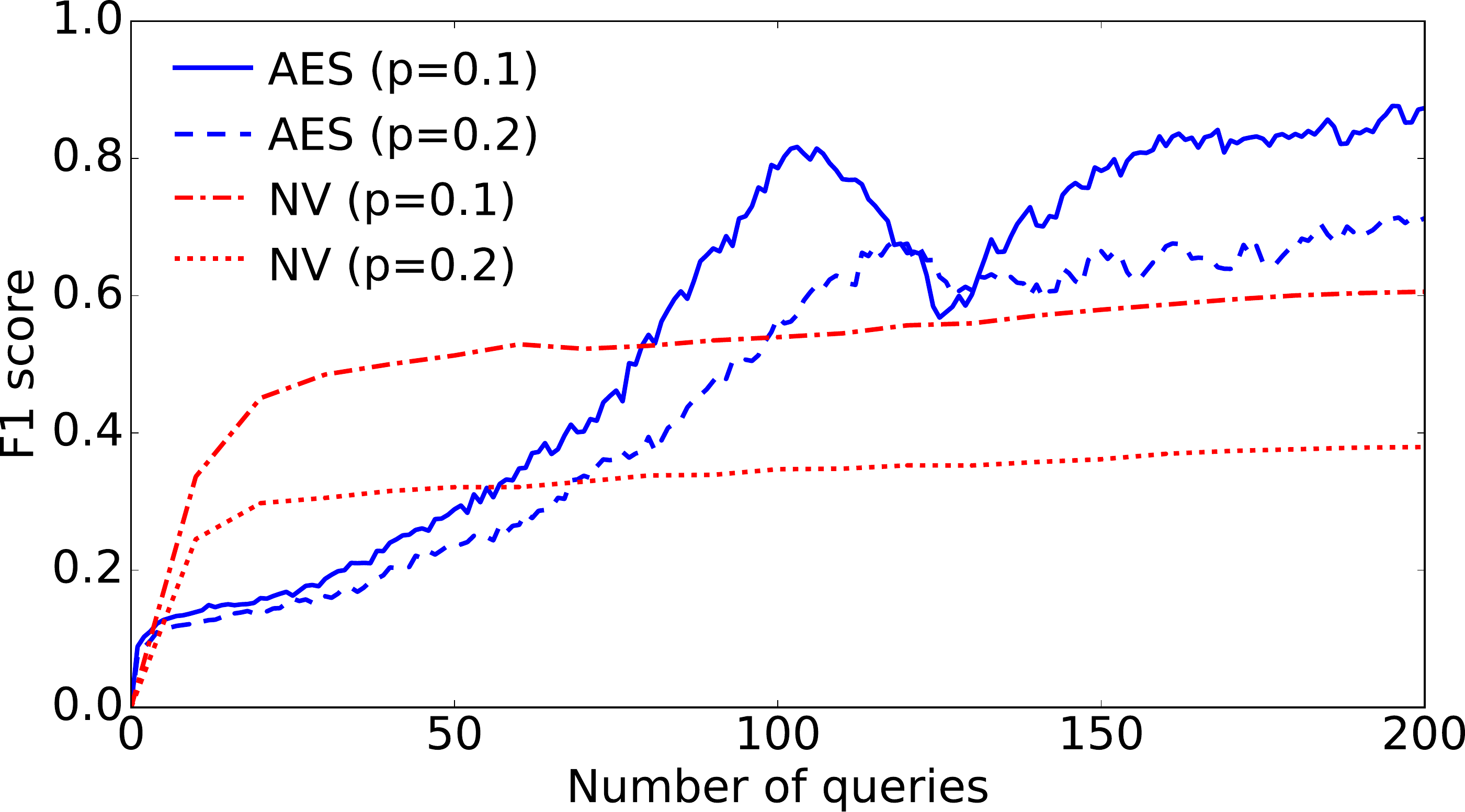}
\label{fig:f1s_bern_hosaki}}
\subfloat[Gaussian noise.]{
\includegraphics[width=0.5\textwidth]{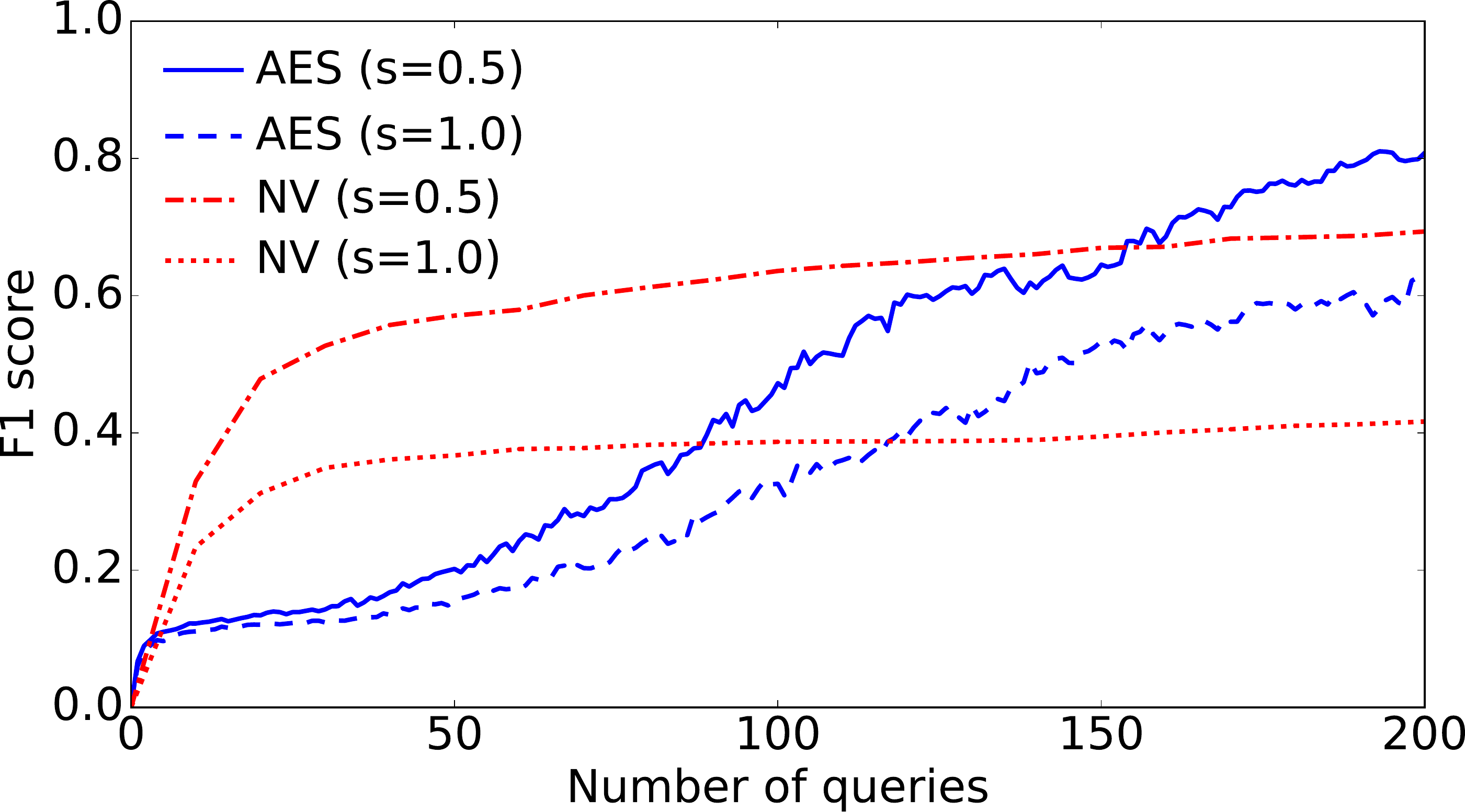}
\label{fig:f1s_gauss_hosaki}}
\caption{AES and NV on the Hosaki example using noisy labels.}
\label{fig:f1s_noisy_hosaki}
\end{figure*}

We use the Hosaki example as an additional 2-dimensional example to demonstrate the performance of our proposed method. Different from the Branin example, the Hosaki example has feasible domains of different scales. Its feasible domains resemble two isolated feasible regions\smoemdash a large ``island'' and a small one (Fig.~\ref{fig:f1s_epsilon_hosaki}). The Hosaki function is
\begin{equation*}
g(\bm{x})=\left(1-8x_1+7x_1^2-\frac{7}{3}x_1^3+\frac{1}{4}x_1^4\right)x_2^2e^{-x_2}
\end{equation*}
We define the label $y=1$ if $\bm{x} \in \{\bm{x}|g(\bm{x})\leq -1, 0<x_1,x_2<5\}$; and $y=-1$ otherwise. 

\begin{table}
\caption{Input space bounds for the NV algorithm and the straddle heuristic (Hosaki example).}
\label{tab:bounds_hosaki}       
\begin{tabular}{llll}
\hline\noalign{\smallskip}
 & Tight & Loose & Insufficient  \\
\noalign{\smallskip}\hline\noalign{\smallskip}
Hosaki & \begin{tabular}{@{}l@{}}$x_1\in[0,6]$,\\$x_2\in[0,5]$\end{tabular} 
		& \begin{tabular}{@{}l@{}}$x_1\in[-2.5,8.5]$,\\$x_2\in[-3,8]$\end{tabular} 
        & \begin{tabular}{@{}l@{}}$x_1\in[1,6]$,\\$x_2\in[0,4.5]$\end{tabular} \\
\noalign{\smallskip}\hline
\end{tabular}
\end{table}

For AES, we set the initial point $\bm{x}^{(0)} = (3,3)$. We use a Gaussian kernel with a length scale $l=0.4$. The test set to compute F1 scores is generated along a $100 \times 100$ grid in the region where $x_1 \in [-3,9]$ and $x_2 \in [-3.5,8.5]$. For NV and straddle, the input space bounds are shown in Tab.~\ref{tab:bounds_hosaki}.

\begin{table}
\caption{Final F1 scores and running time (Hosaki example).}
\label{tab:f1s_time_hosaki}
\begin{tabular}{llrr}
\hline\noalign{\smallskip}
 & & F1 score & Time (s)  \\
\noalign{\smallskip}\hline\noalign{\smallskip}
\multirow{11}{*}{\begin{sideways} Hosaki (200 queries) \end{sideways}} & AES ($\epsilon=0.3,\eta=1.3$)  & $0.95\pm0.003$ & $28.25\pm0.25$ \\
& AES ($\epsilon=0.1,\eta=1.3$)  & $0.94\pm0.004$ & $30.86\pm0.19$ \\
& AES ($\epsilon=0.5,\eta=1.3$)  & $0.95\pm0.002$ & $28.32\pm0.33$ \\
& AES ($\epsilon=0.3,\eta=1.2$)  & $0.94\pm0.003$ & $31.69\pm0.45$ \\
& AES ($\epsilon=0.3,\eta=1.4$)  & $0.96\pm0.002$ & $26.39\pm0.38$ \\
& NV (tight) & $0.95\pm0.003$ & $22.58\pm0.03$ \\
& NV (loose) & $0.93\pm0.004$ & $22.28\pm0.03$ \\
& NV (insufficient) & $0.69\pm0.010$ & $22.27\pm0.03$ \\
& Straddle (tight) & $0.95\pm0.002$ & $16.20\pm0.19$ \\
& Straddle (loose) & $0.88\pm0.005$ & $14.00\pm0.14$ \\
& Straddle (insufficient) & $0.69\pm0.010$ & $16.92\pm0.25$ \\
\noalign{\smallskip}\hline
\end{tabular}
\end{table}

Table~\ref{tab:f1s_time_hosaki} shows the final F1 scores and running time of AES, NV, and the straddle heuristic. Fig.~\ref{fig:hyppara_hosaki} shows the F1 scores and queries under different $\epsilon$ and $\eta$. Fig.~\ref{fig:f1s_nv_hosaki} compares the performance of AES and NV with different boundary sizes. Fig.~\ref{fig:f1s_noisy_hosaki} shows the performance of AES and NV under Bernoulli and Gaussian noise.

\section{Results of Straddle Heuristic}
\label{exp:straddle}

In this section we list experimental results related to the straddle heuristic. Specifically, Fig.~\ref{fig:f1s_straddle} shows straddle's F1 scores and queries using different sizes of input variable bounds, and the comparison with AES. Fig.~\ref{fig:f1s_straddle_noisy} shows the comparison of AES and straddle under noisy labels.

\begin{figure*}
\centering
\subfloat[Branin example.]{
\includegraphics[width=0.5\textwidth]{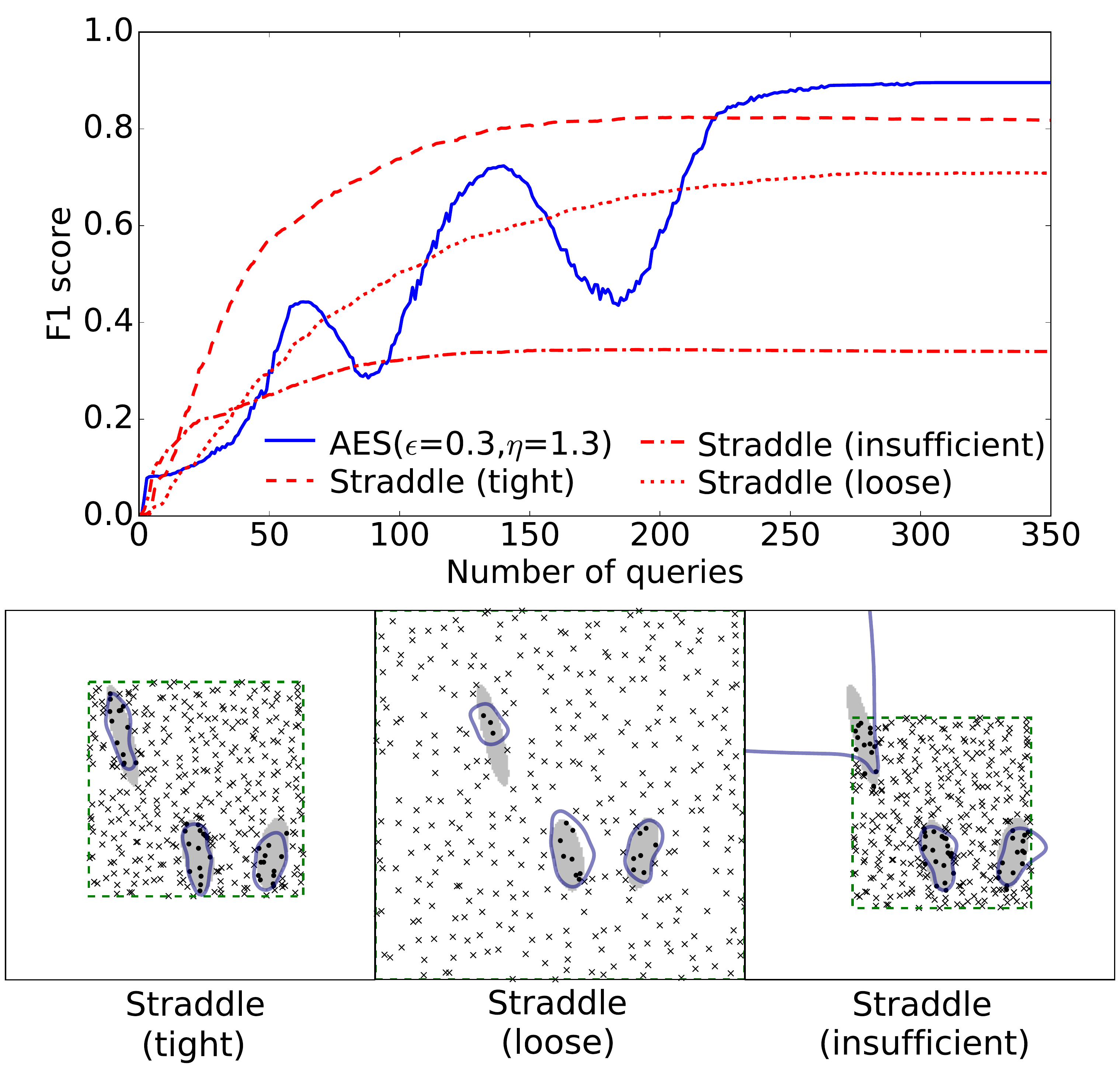}
\label{fig:f1s_straddle_branin}}
\subfloat[Hosaki example.]{
\includegraphics[width=0.5\textwidth]{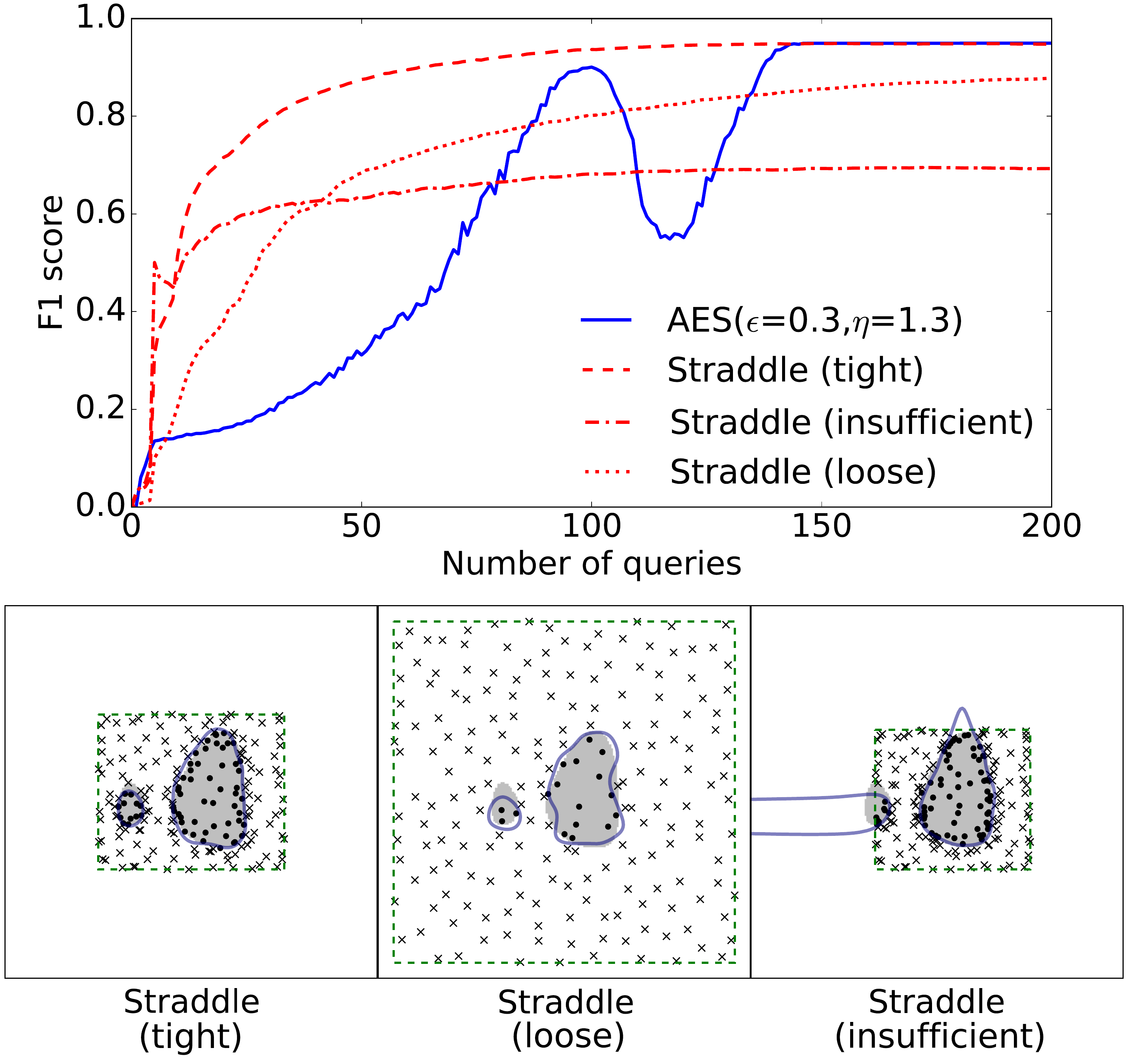}
\label{fig:f1s_straddle_hosaki}}
\caption{AES and straddle (with different input variable bounds).}
\label{fig:f1s_straddle}
\end{figure*}

\begin{figure*}
\centering
\shortstack{
\subfloat[Branin example under Bernoulli noise.]{
\includegraphics[width=0.5\textwidth]{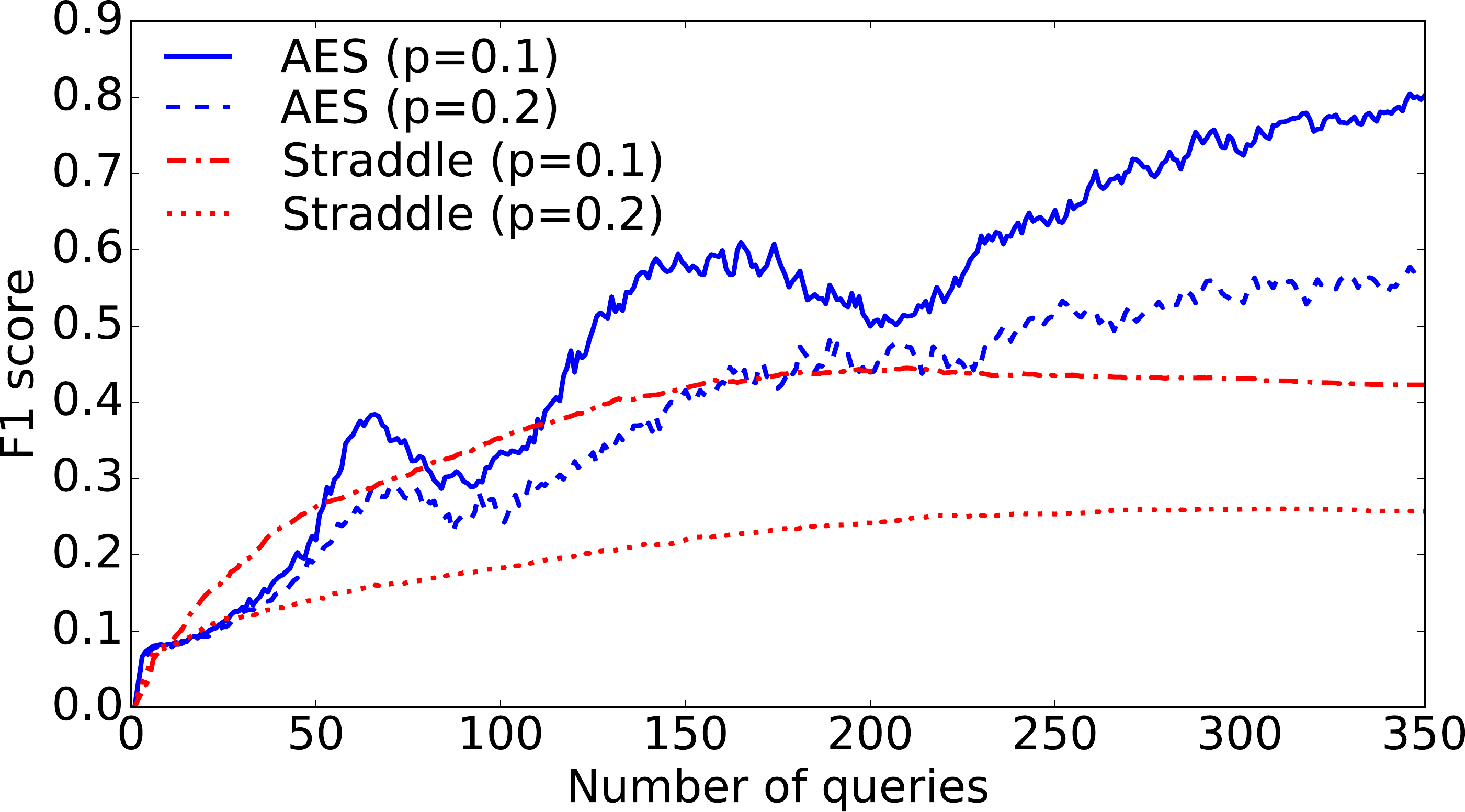}
\label{fig:f1s_straddle_bern_hosaki}}
\subfloat[Branin example under Gaussian noise.]{
\includegraphics[width=0.5\textwidth]{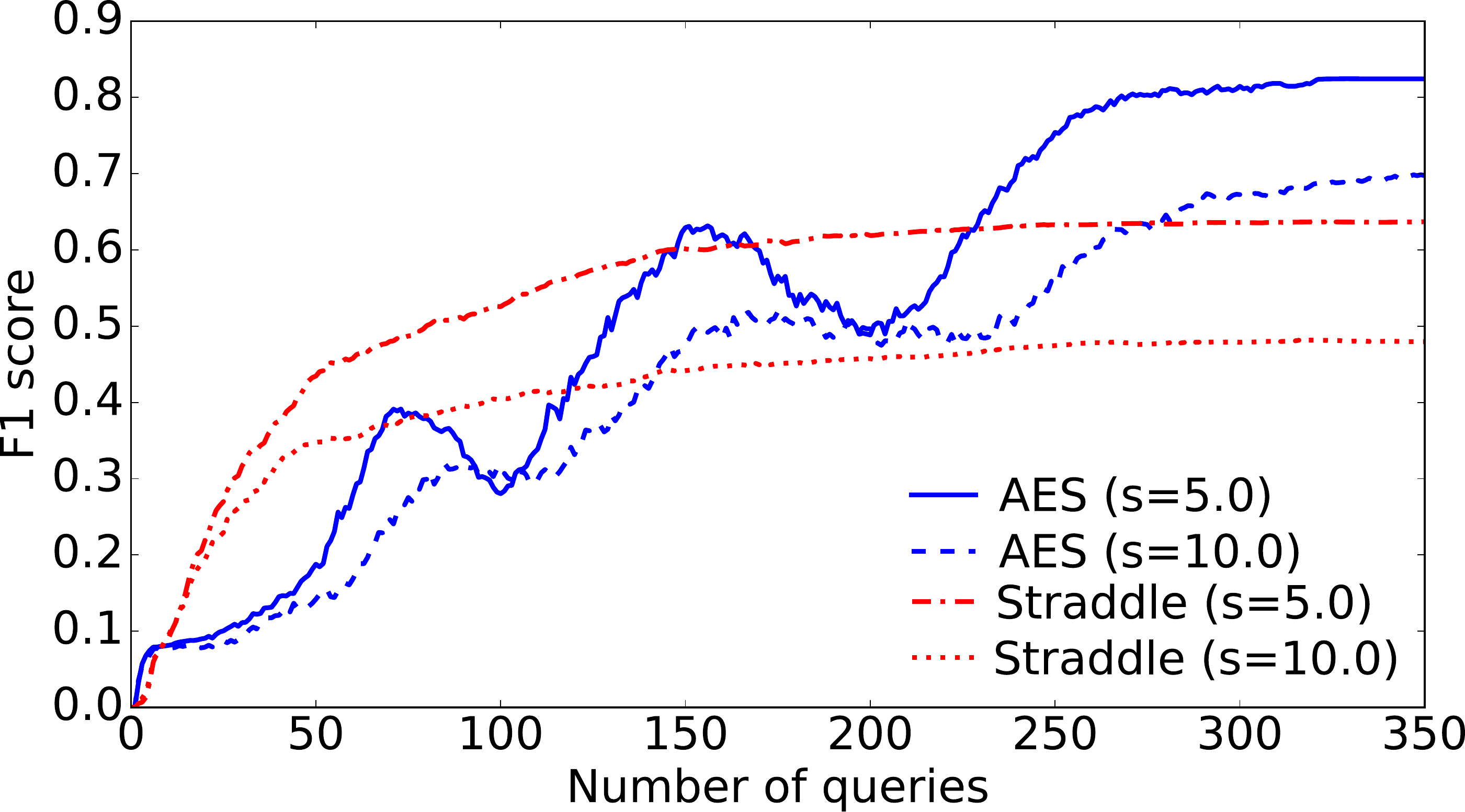}
\label{fig:f1s_straddle_gauss_hosaki}}}
\shortstack{
\subfloat[Hosaki example under Bernoulli noise.]{
\includegraphics[width=0.5\textwidth]{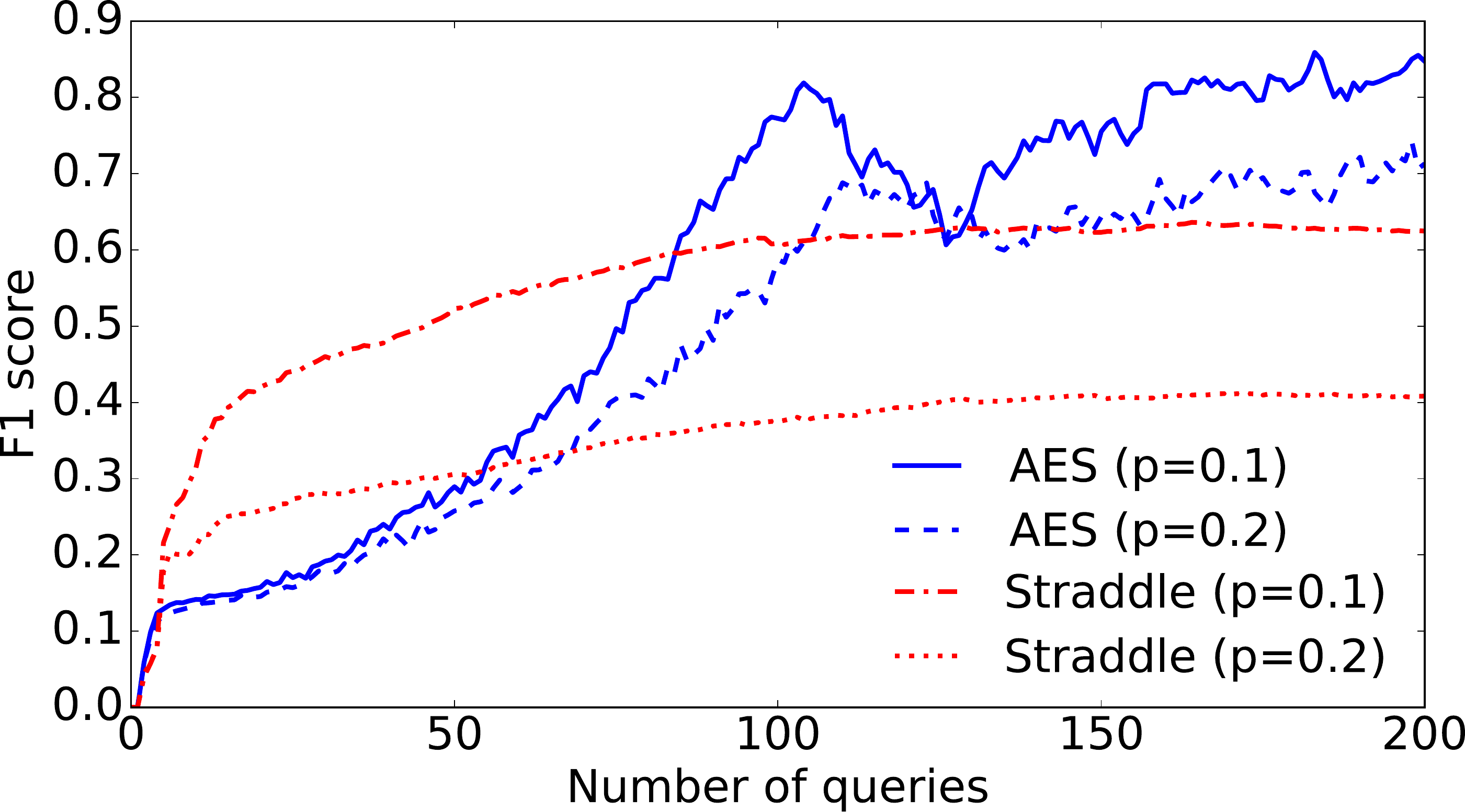}
\label{fig:f1s_straddle_bern_hosaki}}
\subfloat[Hosaki example under Gaussian noise.]{
\includegraphics[width=0.5\textwidth]{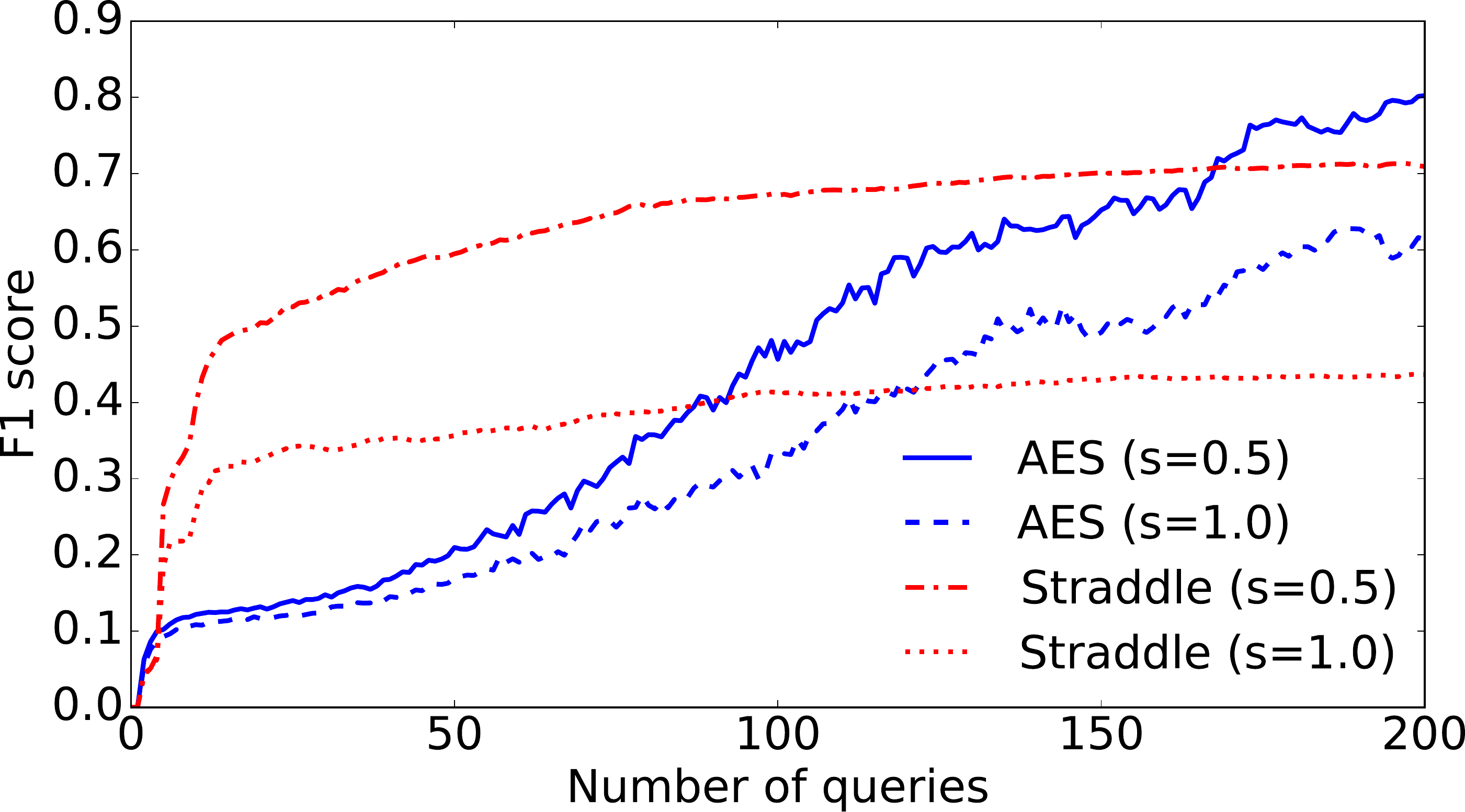}
\label{fig:f1s_straddle_gauss_hosaki}}}
\caption{AES and straddle under noisy labels.}
\label{fig:f1s_straddle_noisy}
\end{figure*}

\end{document}